\documentclass[runningheads]{llncs}
\usepackage[T1]{fontenc}
\usepackage{graphicx}
\usepackage{newfloat}
\usepackage{listings}
\usepackage{algorithm}
\usepackage{algorithmic}
\usepackage{graphicx}
\usepackage{amsmath,amssymb}
\usepackage{mathrsfs}
\newtheorem{prop}{Proposition}

\usepackage{stfloats} 
\usepackage{wrapfig}
\usepackage{miniplot}
\usepackage{multirow}
\usepackage{booktabs}
\usepackage{makecell}
\usepackage{hhline}
\usepackage{multicol}
\usepackage{caption}
\usepackage{subcaption}
\usepackage{graphicx}
\usepackage{wrapfig}
\usepackage[misc]{ifsym}
\usepackage{appendix}
\usepackage{etoc}

\begin{document}
\title{Exploring the Training Robustness of Distributional Reinforcement Learning against Noisy State Observations}
%
%
\author{Ke Sun\inst{1} \and Yingnan Zhao\inst{2} \and
Shangling Jui\inst{3} \and 
Linglong Kong\inst{1 \thanks{Corresponding Author}}(\Letter)}
\authorrunning{K. Sun et al.}
%
\institute{University of Alberta, Edmonton, Alberta, Canada \\
\email{\{ksun6,lkong\}@ualberta.ca}\\
\and Harbin Engineering University, China \\
\email{zhaoyingnan@hrbeu.edu.cn}\\
\and Huawei Kirin Solution\\
\email{jui.shangling@huawei.com}}
\maketitle              
\begin{abstract}
	In real scenarios, state observations that an agent observes may contain measurement errors or adversarial noises, misleading the agent to take suboptimal actions or even collapse while training. In this paper, we study the training robustness of distributional Reinforcement Learning~(RL), a class of state-of-the-art methods that estimate the whole distribution, as opposed to only the expectation, of the total return. Firstly, we validate the contraction of distributional Bellman operators in the State-Noisy Markov Decision Process~(SN-MDP), a typical tabular case that incorporates both random and adversarial state observation noises. In the noisy setting with function approximation, we then analyze the vulnerability of least squared loss in expectation-based RL with either linear or nonlinear function approximation. By contrast, we theoretically characterize the bounded gradient norm of distributional RL loss based on the categorical parameterization equipped with the Kullback–Leibler~(KL) divergence. The resulting stable gradients while the optimization in distributional RL accounts for its better training robustness against state observation noises. Finally, extensive experiments on the suite of environments verified that distributional RL is less vulnerable against both random and adversarial noisy state observations compared with its expectation-based counterpar~\footnote{Code is available in \url{https://github.com/datake/RobustDistRL}.}.
	
\end{abstract}

\section{Introduction}
Learning robust and high-performance policies for continuous state-action
reinforcement learning~(RL) domains is crucial to enable the successful adoption of deep RL in robotics, autonomy, and control problems. However, recent works have demonstrated that deep RL algorithms are vulnerable either to model uncertainties or external disturbances~\cite{huang2017adversarial,pattanaik2017robust,ilahi2020challenges,chen2019adversarial,zhang2020robust,shen2020deep,singh2020improving,guan2020robust}. Particularly, model uncertainties normally occur in a noisy reinforcement learning environment where the agent often encounters systematic or stochastic measurement errors on state observations, such as the inexact locations and velocity obtained from the equipped sensors of a robot. Moreover, external disturbances are normally adversarial in nature. For instance, the adversary can construct adversarial perturbations on state observations to degrade the performance of deep RL algorithms. These two factors lead to noisy state observations that influence the performance of algorithms, precluding the success of RL algorithms in real-world applications.

Existing works mainly focus on improving the robustness of algorithms in the \textit{test environment} with noisy state observations. Smooth
Regularized Reinforcement Learning~\cite{shen2020deep} introduced a regularization to enforce smoothness in the learned policy, and thus improved its robustness against measurement errors in the test environment. Similarly, the State-Adversarial Markov Decision Process~(SA-MDP)~\cite{zhang2020robust} was proposed and the resulting principled policy regularization enhances the adversarial robustness of various kinds of RL algorithms against adversarial noisy state observations. However, both of these works assumed that the agent can access \textit{clean} state observations \textit{during the training}, which is normally not feasible when the environment is inherently noisy, such as unavoidable measurement errors. Hence, the maintenance and formal analysis of policies robust to noisy state observations  \textit{during the training} is a worthwhile area of research.

Recent distributional RL algorithms, e.g., C51~\cite{bellemare2017distributional}, Quantile-Regression DQN (QRDQN)~\cite{dabney2017distributional}, Implicit Quantile Networks~(IQN)~\cite{dabney2018implicit} and Moment-Matching DQN~(MMD)~\cite{nguyen2020distributional}, constantly set new records in Atari games, gaining huge attention in the research community. Existing literature mainly focuses on the performance of distributional RL algorithms, but \textit{other benefits, including the robustness in the noisy environment, of distributional RL algorithms are less studied}. As distributional RL can leverage additional information about the return distribution that captures the uncertainty of the environment more accurately, it is natural to expect that distributional RL with this better representation capability can be less vulnerable to the noisy environment while training, which motivates our research. In this paper, we probe the robustness superiority of distributional RL against various kinds of state observation noises during the training process. Our contributions can be summarized as follows:

\begin{itemize}
	\item \textbf{Tabular setting}. We firstly analyze a systematical noisy setting, i.e., State-Noisy Markov Decision Process~(SN-MDP), incorporating both random and adversarial state observation noises. Theoretically, we derive the convergence of distributional Bellman operator in SN-MDP.
	
	\item \textbf{Function approximation setting}. We elaborate the additional convergence requirement of linear Temporal difference~(TD) when exposed to noisy state observations. To clearly compare with distributional RL, we attribute its robustness advantage to the bounded gradients norms regarding state features based on the categorical parameterization of return distributions. This stable optimization behavior is in contrast to the potentially unbounded gradient norms of expectation-based RL.
	
	\item \textbf{Experiments}. We demonstrate that distributional RL algorithms potentially enjoy better robustness under various types of noisy state observations across a wide range of classical and continual control environments as well as Atari games. Our conclusion facilitates the deployment of  distributional RL algorithms in more practical noisy settings.
	
\end{itemize}

\section{Background: Distributional RL}\label{Section:background}

In the tabular setting without noisy states, the interaction of an agent with its environment can be naturally modeled as a standard Markov Decision Process~(MDP), a 5-tuple ($\mathcal{S}, \mathcal{A}, R, P, \gamma$). $\mathcal{S}$ and $\mathcal{A}$ are the state and action spaces, $P: \mathcal{S} \times \mathcal{A} \times \mathcal{S} \rightarrow [0, 1]$ is the environment transition dynamics, $R: \mathcal{S} \times \mathcal{A} \times \mathcal{S} \rightarrow \mathbb{R}$ is the reward function and $\gamma \in (0,1)$ is the discount factor.

\paragraph{Value Function vs Return Distribution.} Firstly, we denote the \textit{return} as  $Z^{\pi}(s)=\sum_{k=0}^{\infty} \gamma^k r_{t+k+1}$, where $s_t=s$, representing the cumulative rewards following a policy $\pi$, and $r_{t+k+1}$ is reward scalar obtained in the step $t+k+1$. In the algorithm design, classical expectation-based RL normally focuses on \textit{value function} $V^{\pi}(s)$, the expectation of the random variable $Z^{\pi}(s)$:
\begin{eqnarray}\begin{aligned}\label{value_function}
		V^{\pi}(s):=\mathbb{E}\left[Z^{\pi}(s)\right]=\mathbb{E}\left[\sum_{k=0}^{\infty} \gamma^{k} r_{t+k+1} \mid s_{t}=s\right].
\end{aligned}\end{eqnarray}
In distributional RL we focus on the \textit{return distribution}, the full distribution of $Z^{\pi}(s)$, and the \textit{state-action return distribution} $Z^{\pi}(s, a)$ in the control case where $s_t=s, a_t=a$. Both of these distributions can better capture the uncertainty of returns in the MDP beyond just its expectation~\cite{dabney2018implicit,mavrin2019distributional}.

\paragraph{Distributional Bellman Operator.} In expectation-based RL, we update the value function via the Bellman operator $\mathcal{T}^\pi$, while in distributional RL, the updating is applied on the return distribution via the \textit{distributional Bellman operator} $\mathfrak{T}^{\pi}$. To derive $\mathfrak{T}^{\pi}$, we firstly define the transition operator $P^{\pi}:\mathcal{Z}\rightarrow \mathcal{Z}$:
\begin{equation}\begin{aligned}\label{eq_distributionbellman_transition}
		\mathcal{P}^{\pi} Z(s, a) : \overset{D}{=} Z\left(S^{\prime}, A^{\prime}\right), S^{\prime} \sim P(\cdot | s, a), A^{\prime} \sim \pi\left(\cdot | S^{\prime}\right),
\end{aligned}\end{equation}
where we use capital letters $S^{\prime}$ and $A^{\prime}$ to emphasize the random nature of both, and $:\overset{D}{=}$ indicates convergence in distribution. For simplicity, we denote $Z^{\pi}(s, a)$ by $Z(s, a)$. Thus, the distributional Bellman operator $\mathfrak{T}^{\pi}$ is defined as:	
\begin{equation}\begin{aligned}\label{eq_distributionbellman}
		\mathfrak{T}^{\pi} Z(s, a) : \overset{D}{=} R(s, a, S')+\gamma \mathcal{P}^{\pi} Z(s, a).
\end{aligned}\end{equation}
$\mathfrak{T}^{\pi}$ is still a contraction for policy evaluation under  the maximal form of either Wasserstein metric $d_p$~(Quantile Regression distributional RL)~\cite{bellemare2017distributional,dabney2017distributional} or the categorical parameterization equipped with Kullback–Leibler~(KL) divergence~(Categorical distributional RL)~\cite{bellemare2017distributional} over the target and current return distributions. 

\section{Tabular Case: State-Noisy MDP}\label{Section:SNMDP}

In this section, we extend State-Adversarial Markov Decision Process~(SA-MDP)~\cite{zhang2020robust} to a more general State-Noisy Markov Decision Process~(SN-MDP) by incorporating both random and adversarial state noises, and particularly provide a proof of the convergence and contraction of distributional Bellman operators in this setting.

\begin{wrapfigure}[12]{r}{0.6\textwidth}
	\centering
	\includegraphics[width=0.6\textwidth,trim=10 170 250 220,clip]{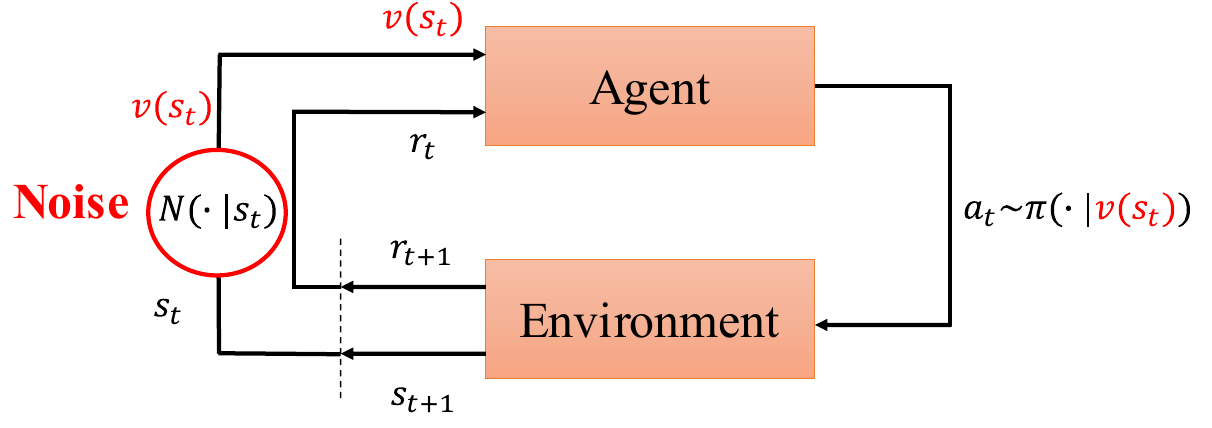}
	\caption{State-Noisy Markov Decision Process. $v(s_t)$ is perturbed by the noise mechanism $N$.}
	\label{Figure_SNMDP}
\end{wrapfigure}

\paragraph{Definitions.} SN-MDP is a 6-tuple ($\mathcal{S}, \mathcal{A}, R, P, \gamma, N$), as exhibited in Figure~\ref{Figure_SNMDP}, where the noise generating mechanism $N(\cdot|s)$ maps the state from $s$ to $v(s)$ using either random or adversarial noise with the Markovian and stationary probability $N(v(s)|s)$. It is worthwhile to note that the explicit definition of the noise mechanism $N$ here is based on discrete state transitions, but the analysis can be naturally extended to the continuous case if we let the state space go to infinity. Moreover, let $\mathcal{B}(s)$ be the set that contains the allowed noise space for the noise generating mechanism $N$, i.e., $v(s)\in \mathcal{B}(s)$.	

Following the setting in \cite{zhang2020robust}, we only manipulate state observations but do not change the underlying environment transition dynamics based on $s$ or the agent's actions directly. As such, our SN-MDP is more suitable to model the random measurement error, e.g., sensor errors and equipment inaccuracies, and adversarial state observation perturbations in safety-critical scenarios. This setting is also aligned with many adversarial attacks on state observations~\cite{huang2017adversarial,lin2017tactics}. The following contractivity analysis regarding value function or distribution is directly based the state $s$ rather than $v(s)$ as it is more natural and convenient to capture the uncertainty of MDP.

\subsection{Analysis of SN-MDP for Expectation-based RL}
We define the value function $\tilde{V}_{\pi \circ N}$ given $\pi$ in SN-MDP. The Bellman Equations regarding the new value function $\tilde{V}_{\pi \circ N}$ are given by:	
\begin{equation}\label{eq:BE_SNMDP} \begin{aligned}
		\tilde{V}_{\pi \circ N}(s) = \sum_{a} \sum_{v(s)} N(v(s)|s) \pi(a|v(s)) \sum_{s^\prime}p(s^{\prime}|s,a)  \left[R(s,a,s^{\prime})+\gamma \tilde{V}_{\pi \circ N}(s^\prime) \right],
\end{aligned}\end{equation}
where the random noise transits $s$ into $v(s)$ with a certain probability and the adversarial noise is the special case of $N(v(s)|s)$ where $N(v^{*}(s)|s)=1$ if $v^{*}(s)$ is the optimal adversarial noisy state given $s$, and $N(v(s)|s)=0$ otherwise. We denote Bellman operators under random noise mechanism $N^r(\cdot|s)$ and adversarial noise mechanism $N^*(\cdot|s)$ as $\mathcal{T}^{\pi}_r$ and $\mathcal{T}^{\pi}_a$, respectively. This implies that $\mathcal{T}^{\pi}_r \tilde{V}_{\pi \circ N} = \tilde{V}_{\pi \circ N^r}$ and $\mathcal{T}^{\pi}_a \tilde{V}_{\pi \circ N} = \tilde{V}_{\pi \circ N^*}$. We extend Theorem 1 in \cite{zhang2020robust} to both random and adversarial noise scenarios, and immediately obtain that both $\mathcal{T}^{\pi}_r$ and $\mathcal{T}^{\pi}_a$ are contraction operators in SN-MDP. We provide a rigorous description in Theorem~\ref{theorem:SNMDP} with the proof in Appendix~\ref{appendix:theorem:SNMDP}. 

The insightful and pivotal conclusion from Theorem~\ref{theorem:SNMDP} is $\mathcal{T}^{\pi}_a \tilde{V}_{\pi \circ N} = \min_{N} \tilde{V}_{\pi \circ N}$. This implies that the adversary attempts to minimize the value function, forcing the agent to select the worse-case action among the allowed transition probability space $N(\cdot|s)$ for each state $s$. The crux of the proof is that Bellman updates in SN-MDP result in the convergence to the value function for another ``merged'' policy $\pi^\prime$ where $\pi^{\prime}(a|s)=\sum_{v(s)} N(v(s)|s) \pi(a|v(s))$. Nevertheless, the converged value function corresponding to the merged policy might be far away from that for the original policy $\pi$, which is more likely to worsen the performance of RL algorithms.

\subsection{Analysis of SN-MDP in distributional RL}

In the SN-MDP setting for distributional RL, the new distributional Bellman equations use new transition operators in place of  $\mathcal{P}^{\pi}$ in Eq.~\ref{eq_distributionbellman_transition}. The new transition operators $\mathcal{P}^{\pi}_r$ and $\mathcal{P}^{\pi}_a$, for the random and adversarial settings, are defined as:	
\begin{equation}\begin{aligned}
		\mathcal{P}^{\pi}_r Z_N(s, a) : &\overset{D}{=}  Z_{N^r}(S^{\prime}, A^{\prime}), A^{\prime} \sim \pi(\cdot | V(S^{\prime})), \text{and} \\ \ \mathcal{P}^{\pi}_a Z_N(s, a) : &\overset{D}{=} Z_{N^*}(S^{\prime}, A^{\prime}), A^{\prime} \sim \pi(\cdot | V^*(S^{\prime})),  
\end{aligned}\end{equation}
where $V(S^\prime)\sim N^r(\cdot|S^{\prime})$ is the state random variable after the transition, and $V^*(S^{\prime})$ is attained from $N^*(\cdot|S^{\prime})$ under the optimal adversary. Besides, $S^{\prime} \sim P(\cdot | s, a)$. Therefore, the corresponding new distributional Bellman operators $\mathfrak{T}^{\pi}_r$ and $\mathfrak{T}^{\pi}_a$ are formulated as:
\begin{equation}\label{eq:distributionalBE_SNMDP}\begin{aligned}
		\mathfrak{T}^{\pi}_r Z_N(s, a) : &\overset{D}{=} R(s, a, S')+\gamma \mathcal{P}^{\pi}_r Z_N(s, a),  \text{and} \\ \ \mathfrak{T}^{\pi}_a Z_N(s, a) : &\overset{D}{=} R(s, a, S')+\gamma \mathcal{P}^{\pi}_a Z_N(s, a).
\end{aligned}\end{equation}
In this sense, four sources of randomness define the new compound distribution in the SN-MDP: (1) randomness of reward, (2) randomness in the new environment transition dynamics $\mathcal{P}^{\pi}_r$ or $\mathcal{P}^{\pi}_a$ that additionally includes (3) the stochasticity of the noisy transition $N$, and (4) the random next-state return distribution $Z(S^{\prime}, A^{\prime})$. As our first theoretical contribution, we now show that the new derived distribution Bellman Operators defined in Eq.~\ref{eq:distributionalBE_SNMDP} in SN-MDP setting are convergent and contractive for policy evaluation in Theorem~\ref{theorem:SNMDP_dRL}.

\begin{theorem}\label{theorem:SNMDP_dRL}(Convergence and Contraction of Distributional Bellman operators in the SN-MDP) Given a policy $\pi$, we define the distributional Bellman operators $\mathfrak{T}^{\pi}_r$ and $\mathfrak{T}^{\pi}_a$  in Eq.~\ref{eq:distributionalBE_SNMDP}, and consider the Wasserstein metric $d_p$, the following results hold.
	
	\noindent	(1) $\mathfrak{T}^{\pi}_r$ is a contraction under the maximal form of $d_p$.
	
	\noindent	(2) $\mathfrak{T}^{\pi}_a$ is also a contraction under the maximal form of $d_p$ when $p=1$, following the greedy adversarial rule, i.e., $N^*(\cdot|s^\prime) = \mathop{\arg\min}_{N(\cdot|s^\prime)} \mathbb{E}\left[Z(s^\prime, a^\prime)\right]$ where $a^\prime\sim\pi(\cdot|V(s^\prime))$ and $V(s^\prime)\sim N(\cdot|s^\prime)$.
	
\end{theorem} 
We provide the proof in Appendix~\ref{appendix:theorem:SNMDP_dRL}. Similar to the convergence conclusions in classical RL, Theorem~\ref{theorem:SNMDP_dRL}  justified that distributional RL is also capable of converging in this SN-MDP setting. The contraction and convergence of distributional Bellman operators in the SN-MDP is one of our main contributions. This result allows us to deploy distributional RL algorithms comfortably in the tabular setting even with noisy state observations. 

\section{Function Approximation Case}\label{Section:Flexible}

In the tabular case, both expectation-based and distributional RL have convergence properties. However, in the function approximation case, we firstly show linear TD requires more conditions for the convergence, and point out the vulnerability of expectation-based RL against noisy states even under the bounded rewards assumption. In contrast, we analyze that distributional RL with the categorical representation for the return distributions, is more robust against noisy state observations due to its bounded gradient norms.

\subsection{Convergence of Linear TD under Noisy States}

In classical RL with function approximation, the value estimator $\hat{v}: \mathcal{S}\times \mathbb{R}^d\rightarrow \mathbb{R}$ parameterized by $\mathbf{w}$ is expressed as $\hat{v}(s, \mathbf{w})$. The objective function is \textit{Mean Squared Value Error}~\cite{sutton2018reinforcement} denoted as $\overline{\mathrm{VE}}$:
\begin{equation}\begin{aligned}\label{eq:VE}
		\overline{\mathrm{VE}}(\mathbf{w}) \doteq \sum_{s \in \mathcal{S}} \mu(s)\left[v_{\pi}(s)-\hat{v}(s, \mathbf{w})\right]^{2},
\end{aligned}\end{equation}
where $\mu$ is the state distribution. In linear TD, the value estimate  is formed simply as the inner product between state features $\mathbf{x}(s)$ and weights $\mathbf{w}\in \mathbb{R}^d$, given by $\hat{v}(s, \mathbf{w}) \stackrel{\text { def }}{=} \mathbf{w}^{\top} \mathbf{x}(s)$. At each step, the state feature can be rewritten as $\mathbf{x}_{t} \stackrel{\text { def }}{=} \mathbf{x}\left(S_{t}\right) \in \mathbb{R}^{d}$. Thus, the TD update at step $t$ is:	
\begin{equation}\label{eq:TD_update}\begin{aligned}
		\mathbf{w}_{t+1} \leftarrow \mathbf{w}_{t} + \alpha_t (R_{t+1} + \gamma \mathbf{w}_t^{\top} \mathbf{x}_{t+1} - \mathbf{w}_t^{\top} \mathbf{x}_{t}) \mathbf{x}_{t}
\end{aligned}\end{equation}		
where $\alpha_t$ is the step size at time $t$. Once the system has reached the steady state for any $\mathbf{w}_{t}$, then the expected next weight vector can be written as $\mathbb{E}[\mathbf{w}_{t+1}|\mathbf{w}_{t}]=\mathbf{w}_{t} + \alpha_t (\mathbf{b}-\mathbf{A} \mathbf{w}_{t})$, where $\mathbf{b}=\mathbb{E}(R_{t+1}\mathbf{x}_t)\in \mathbb{R}^d$ and $\mathbf{A} \doteq \mathbb{E}\left[\mathbf{x}_{t} d_t^{\top}\right] \in \mathbb{R}^{d \times d}$. The TD fixed point $\mathbf{w}_{\text{TD}}$ to the system satisfies $\mathbf{A} \mathbf{w}_{\text{TD}} = \mathbf{b}$. From \cite{sutton2018reinforcement}, we know that the matrix $\mathbf{A}$ determines the convergence in the linear TD setting. In particular, $\mathbf{w}_t$ converges with probability one to the TD fixed point if $\mathbf{A}$ is positive definite. However, if we add state noises $\eta$ on $\mathbf{x}_t$ in Eq.~\ref{eq:TD_update}, the convergence condition will be different. As shown in Theorem~\ref{theorem:TD0}~(a more formal version with the proof is given in Appendix~\ref{appendix:TD}), linear TD under noisy state observations requires additional positive definiteness condition.

\begin{theorem}\label{theorem:TD0}(Covergence Conditions for Linear TD under Noisy State Observations) Define $\mathbf{P}$ as the $|\mathcal{S}| \times |\mathcal{S}|$ matrix forming from the state transition probability $p(s^\prime|s)$, $\mathbf{D}$ as the $|\mathcal{S}| \times |\mathcal{S}|$ diagonal matrix with $\mu(s)$ on its diagonal, and $\mathbf{X}$ as the $|\mathcal{S}| \times d$ matrix with $\mathbf{x}(s)$ as its rows, and $\mathbf{E}$ is the $|\mathcal{S}| \times d$ perturbation matrix with each perturbation vector $\mathbf{e}(s)$ as its rows. $\mathbf{w}_t$ converges to TD fixed point \textbf{when both $\mathbf{A}$ and $(\mathbf{X}+\mathbf{E})^\top \mathbf{D} \mathbf{P} \mathbf{E}$ are positive definite.}
	
\end{theorem} 

However, directly analyzing the convergence conditions of distributional linear TD and then comparing with them in Theorem~\ref{theorem:TD0} for classical linear TD is tricky in theory.	As such, we additionally provide a sensitivity comparison of both expectation-based and distributional RL through the lens of their gradients regarding state features as follows.

\subsection{Vulnerability of Expectation-based RL}\label{sec:vulnerability}

We reveal that the vulnerability of expectation-based RL can be attributed to its unbounded gradient characteristics in both linear and nonlinear approximation settings.

\paragraph{Linear Approximation Setting.} To solve the \textit{weighted} least squared minimization in Eq.~\ref{eq:VE}, we leverage Stochastic Gradient Descent~(SGD) on the empirical version of $\overline{\mathrm{VE}}$, which we denote as $g_{\overline{\mathrm{VE}}}$. \textit{We focus on the gradient norm of $g_{\overline{\mathrm{VE}}}$ regarding the state features $\mathbf{x}(s)$~(or $\mathbf{x}_{t}$) as the gradient of loss w.r.t state observations is highly correlated with the sensitivity or robustness of algorithms against the noisy state observations.} For a fair comparison with distributional RL in next section, we additionally bound the norm of $\mathbf{w}$, i.e., $\Vert \mathbf{w} \Vert \leq l$, which can also be easily satisfied by imposing $\ell_1$ or $\ell_2$ regularization. Therefore, we derive the upper bound of gradient norm of $g_{\overline{\mathrm{VE}}}$ as
\begin{equation}\begin{aligned}\label{eq:expectationRL_linear}
		\Vert \frac{\partial g_{\overline{\mathrm{VE}}(\mathbf{w})}}{\partial \mathbf{x}_t}  \Vert = | U_t-\mathbf{w}_{t}^{\top} \mathbf{x}_{t}| \Vert \mathbf{w}_{t} \Vert \leq | U_t-\mathbf{w}_{t}^{\top} \mathbf{x}_{t}| l,
\end{aligned}\end{equation}
where the target $U_t$ can be either an unbiased estimate via Monte Carlo method with $U_t=\sum_{k=0}^\infty \gamma^k r_{t+k+1}$, or a biased estimate via TD learning with $U_t=r_{t+1}+\gamma \mathbf{w}_{t}^{\top} \mathbf{x}_{t+1}$. However, this upper bound $| U_t-\mathbf{w}_{t}^{\top} \mathbf{x}_{t}| l$ heavily depends on the perturbation size or noise strength. Even under the bounded rewards assumption, i.e., $r \in [R_{\text{min}}, R_{\text{max}}]$, we can bound $U_t$ as $U_t=\sum_{k=0}^\infty \gamma^k r_{t+k+1} \in [\frac{R_{\text{min}}}{1-\gamma}, \frac{R_{\text{max}}}{1-\gamma}]$. However, this upper bound can be arbitrarily large if we have no restriction on the noise size, leading to a potentially huge vulnerability against state observation noises.

\paragraph{Nonlinear Approximation Setting.} The potentially unbounded gradient norm issue of expectation-based RL in the linear case still remains in the nonlinear approximation setting. We express the value estimate $\hat{v}$ as $\hat{v}(s; \mathbf{w}, \theta)=\phi_{\mathbf{w}}(\mathbf{x}(s))^\top\theta$, where $\phi_{\mathbf{w}}(\mathbf{x}(s))$ is the representation vector of the state feature $\mathbf{x}(s)$ in the penultimate layer of neural network-based value function approximator. Correspondingly, $\theta$ would be the parameters in the last layer of this value neural network. We simplify $\phi_{\mathbf{w}}(\mathbf{x}(s))_t$ as $\phi_{\mathbf{w}, t}$ in the step $t$ update. As such, akin to the linear case, we derive the upper bound of gradient norm of $g_{\overline{\mathrm{VE}}}$ as 
\begin{equation}\begin{aligned}\label{eq:expectationRL_nonlinear}
		\Vert \frac{\partial g_{\overline{\mathrm{VE}}(\mathbf{w}, \theta)}}{\partial \mathbf{x}_t}  \Vert &= | U_t-\phi_{\mathbf{w}, t}^\top\theta_t| \Vert \nabla_{\mathbf{x}_t}\phi_{\mathbf{w}, t}^\top\theta_t \Vert \leq | U_t-\phi_{\mathbf{w}, t}^\top\theta_t| lL,
\end{aligned}\end{equation}
where we assume the function $\phi_{\mathbf{w}}(\cdot)$ is $L$-Lipschitz continuous regarding its input state feature $\mathbf{x}(s)$, and $\Vert \theta \Vert \leq l$ as well for a fair comparison with distributional RL. It turns out that $| U_t-\phi_{\mathbf{w}, t}^\top\theta_t|$ still depends on the perturbation size, and can be still arbitrarily large if there is no restriction on the noise size. In contrast, we further show that gradient norms in distributional RL can be upper bounded \textbf{regardless of the perturbation size or noise strength}.

\subsection{Robustness Advantage of Distributional RL}\label{sec:lipschitz}

We analyze the distributional loss in distributional RL can potentially lead to bounded gradient norms regarding state features regardless of the perturbation size or noise strength, yielding its training robustness against state noises. In distributional RL our goal is to minimize a distribution loss $\mathcal{L} \left(Z_{\mathbf{w}}, \mathfrak{T} Z_{\mathbf{w}} \right)$ between the current return distribution of $Z_{\mathbf{w}}$ and its target return distribution of $\mathfrak{T} Z_{\mathbf{w}}$.

Our robustness analysis is based on the categorical parameterization~\cite{imani2018improving} on the return distribution with the KL divergence, a typical choice also used in the first distributional RL branch, i.e., C51~\cite{bellemare2017distributional}. Specifically, we uniformly partition the support of $Z_{\mathbf{w}}(s)$ into $k$ bins, and let the histogram function $f: \mathcal{X} \rightarrow[0,1]^{k}$ provide k-dimensional vector $f(\mathbf{x}(s))$ of the coefficients indicating the probability the target is in that bin given $\mathbf{x}(s)$. We use \textit{softmax} to output the $k$ probabilities of $f(\mathbf{x}(s))$. Therefore,  the categorical distributional RL loss $\mathcal{L}(Z_{\mathbf{w}}(s), \mathfrak{T} Z_{\mathbf{w}}(s))$, denoted as $\mathcal{L}_\mathbf{w}$, equipped with KL divergence between $Z_{\mathbf{w}}$ and $\mathfrak{T} Z_{\mathbf{w}}$ can be simplified as
\begin{equation}\begin{aligned}\label{eq:histogram}
		\mathcal{L}(Z_{\mathbf{w}}(s), \mathfrak{T} Z_{\mathbf{w}}(s)) &\propto -\sum_{i=1}^{k} p_{i} \log f_{i}^{\mathbf{w}}(\mathbf{x}(s)),
\end{aligned}\end{equation} 
where we use $\mathbf{w}$ to parameterize the function $f$ in the distributional loss  $\mathcal{L}_\mathbf{w}$, and the target probability $p_i$ is the cumulative probability increment of target distribution $\mathfrak{T} Z_{\mathbf{w}}$ within the $i$-th bin. Detailed derivation about the simplification of categorical distributional loss is in Appendix~\ref{appendix:lipschitz}. 

\paragraph{Linear Approximation Setting.} We leverage $\mathbf{x}(s)^{\top} \mathbf{w}_{i}$ to express the $i$-th output of $f$, i.e., $f_{i}(\mathbf{x}(s))=\exp \left(\mathbf{x}(s)^{\top} {\mathbf{w}_{i}}\right) / \sum_{j=1}^{k} \exp \left(\mathbf{x}(s)^{\top} \mathbf{w}_{j}\right)$, where all parameters are $\mathbf{w}=\{\mathbf{w}_1, ..., \mathbf{w}_{k}\}$. Based on this categorical distributional RL loss, we obtain Proposition~\ref{theorem:dRL_bounded_linear}~(proof in Appendix~\ref{appendix:lipschitz}), revealing that value-based categorical distributional RL loss can result in bounded gradient norms regarding state features $\mathbf{x}(s)$.	

\begin{prop}\label{theorem:dRL_bounded_linear}(Gradient Property of distributional RL in Linear Approximation) Consider the categorical distributional RL loss $\mathcal{L}_{\mathbf{w}}$ in Eq.~\ref{eq:histogram} with the linear approximation. Assume $\Vert \mathbf{w}_i \Vert \leq l$ for $\forall i=1,..,k$, then $\left\|\frac{\partial \mathcal{L}_\mathbf{w}}{\partial \mathbf{x}(s)}  \right\| \leq kl$.
\end{prop} 

In contrast with the unbounded gradient norm in Eq.~\ref{eq:expectationRL_linear} of classical RL, we have a restricted upper bound in distributional RL loss with a linear approximator, i.e., $kl$, which is independent of the perturbation size or noise strength.

\paragraph{Nonlinear Approximation Setting.} Similar to the nonlinear form in classical expectation-based RL as analyzed in Section~\ref{sec:vulnerability},  we express the $i$-th output probabilities of $f(\mathbf{x}(s))$ as $f_i^{\mathbf{w}, \theta}(\mathbf{x}(s))=\exp \left(\phi_{\mathbf{w}}(\mathbf{x}(s))^\top \theta_i\right) / \sum_{j=1}^{k} \exp \left(\phi_{\mathbf{w}}(\mathbf{x}(s))^\top \theta_j\right)$ in distributional RL, where the last layer parameter $\theta=\{\theta_1, ..., \theta_{k}\}$ and $\phi_{\mathbf{w}}(\mathbf{x}(s))$ is still the representation vector of $\mathbf{x}(s)$. In Proposition~\ref{theorem:dRL_bounded_nonlinear}, we can still attain a bounded gradient norm of distributional RL loss in the nonlinear case.

\begin{prop}\label{theorem:dRL_bounded_nonlinear}(Gradient Property of distributional RL in Nonlinear Approximation) Consider the categorical distributional RL loss $\mathcal{L}_{\mathbf{w}, \theta}$ in Eq.~\ref{eq:histogram} with the nonlinear approximation. Assume $\Vert \theta_i \Vert \leq l$ for $\forall i=1,..,k$ and $\phi_{\mathbf{w}}(\cdot)$ is $L$-Lipschitz continuous, then $\left\|\frac{\partial \mathcal{L}_{\mathbf{w}, \theta}}{\partial \mathbf{x}(s)}  \right\| \leq klL$.
\end{prop} 

Please refer to Appendix~\ref{appendix:lipschitz} for the proof. For a fair comparison with nonlinear approximation in classical RL, we still assume the function $\phi_{\mathbf{w}}(\cdot)$ to be $L$-Lipschitz continuous and $\Vert \theta_i \Vert \leq l$. \textit{Interestingly, the bounded gradient norm of the distributional RL loss is independent of the noise size, which is in stark contrast to the potentially unrestricted gradients in classical RL in Eq.~\ref{eq:expectationRL_nonlinear} that heavily depends on the noise size.} Based on Theorems~\ref{theorem:dRL_bounded_linear} and \ref{theorem:dRL_bounded_nonlinear}, we conclude that the bounded gradient behaviors of distributional RL could reduce its sensitivity to state noises, and thus mitigate the interference of the state observation noises compared with expectation-based RL, potentially leading to better training robustness.

\paragraph{Extension of TD Convergence and Sensitivity Analysis.} As supplementary, we also conduct the analysis on different TD convergence conditions under the unbalanced perturbations on either the current or next state observations. Please refer to Theorem~\ref{theorem:TD} with the detailed explanation in Appendix~\ref{appendix:TD}. In addition, we also conduct a sensitivity analysis from the perspective of the \textit{influence function} to characterize the impact of state noises on an estimator. We provide the details in Theorem~\ref{theorem:IF} of Appendix~\ref{appendix:IF}.

\section{Experiments}\label{Section:experiment}

We make a comparison between expectation-based and distributional RL algorithms against various noisy state observations across \textbf{classical and continuous control environments as well as Atari games, including Cartpole and MountainCar~(classical control), Ant, Humanoidstandup and Halfcheetah~(continuous control), Breakout and Qbert~(Atari games)}. For the continuous control environment, we use Soft Actor Critic~\cite{haarnoja2018soft} and Distributional Soft Actor Critic~\cite{ma2020dsac} with C51 as the critic loss and thus we denote them as SAC and DAC~(C51), respectively. For the classical control and Atari games, we utilize DQN~\cite{mnih2015human} as the baseline, and C51~\cite{bellemare2017distributional}, QRDQN~\cite{dabney2017distributional} as its distributional counterparts. The training robustness of C51 could be consistent with our theoretical analysis, while QRDQN, the more commonly-used one, is also applied to  demonstrate that our robustness analysis can also be empirically applicable to broader distributional RL algorithms.

\paragraph{Implementation and Experimental Setup.} For the continuous control environment, we modified our algorithm based on released implementation of \cite{ma2020dsac}. For classical control and Atari games, we followed the procedure in \cite{ghiassian2020gradient,zhang2019quota}. All the experimental settings, including parameters, are identical to the distributional RL baselines implemented by \cite{deeprl,dabney2017distributional}. We perform 200 runs on both Cart Pole and Mountain Car and 3 runs on Breakout and Qbert. Reported results are averaged with shading indicating the standard error. The learning curve is smoothed over a window of size 10 before averaging across runs. Please refer to Appendix~\ref{appendix:setting} for more details about the experimental setup.

\paragraph{Evaluation of Training Robustness.} Due to final performance difference between expectation-based and distributional RL, for a fair comparison we calculate \textit{the ratio between final average returns under random or adversarial state noises with different noise strengths and the original level without any state noises}. This ratio can be used to measure the robustness maintenance after the agent gets exposed to noisy state observations. 

\begin{figure*}[t!]
	\centering
	\begin{subfigure}[t]{0.42\textwidth}
		\centering
		\includegraphics[width=\textwidth,trim=20 0 0 30,clip]{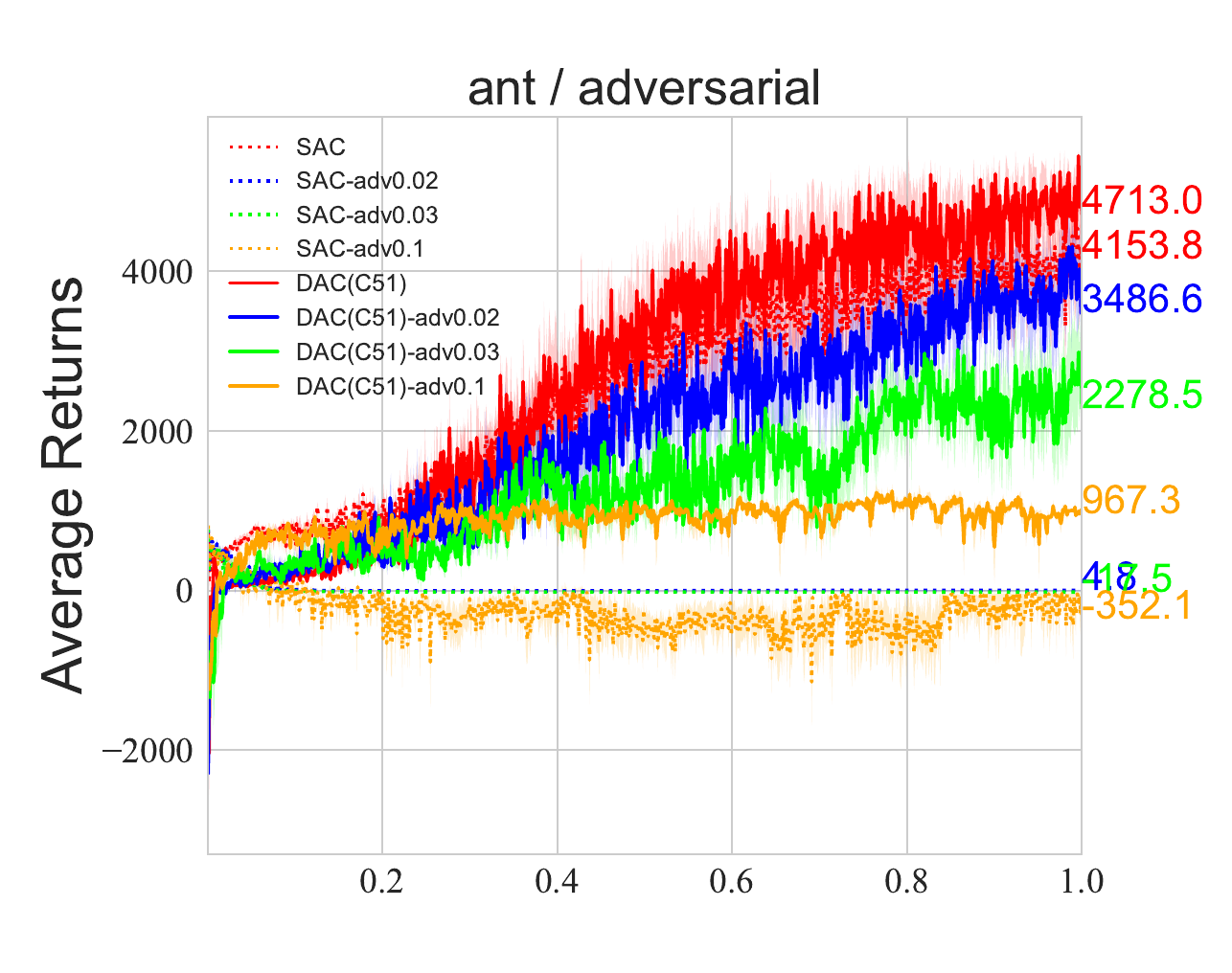}
	\end{subfigure}
	\begin{subfigure}[t]{0.46\textwidth}
		\centering
		\includegraphics[width=\textwidth,trim=0 0 0 30,clip]{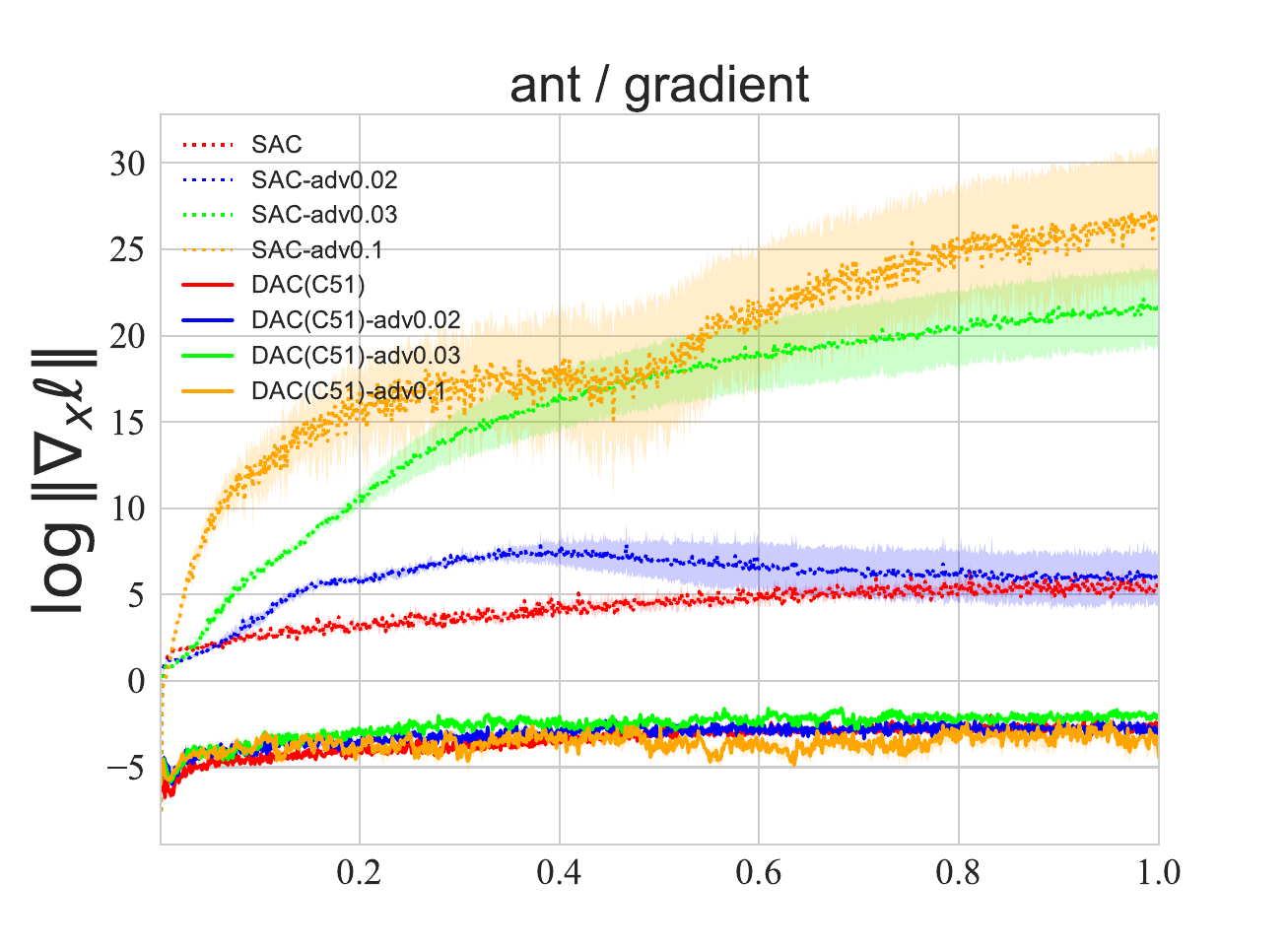}
	\end{subfigure}
	\begin{subfigure}[t]{0.46\textwidth}
		\centering
		\includegraphics[width=\textwidth,trim=10 0 0 20,clip]{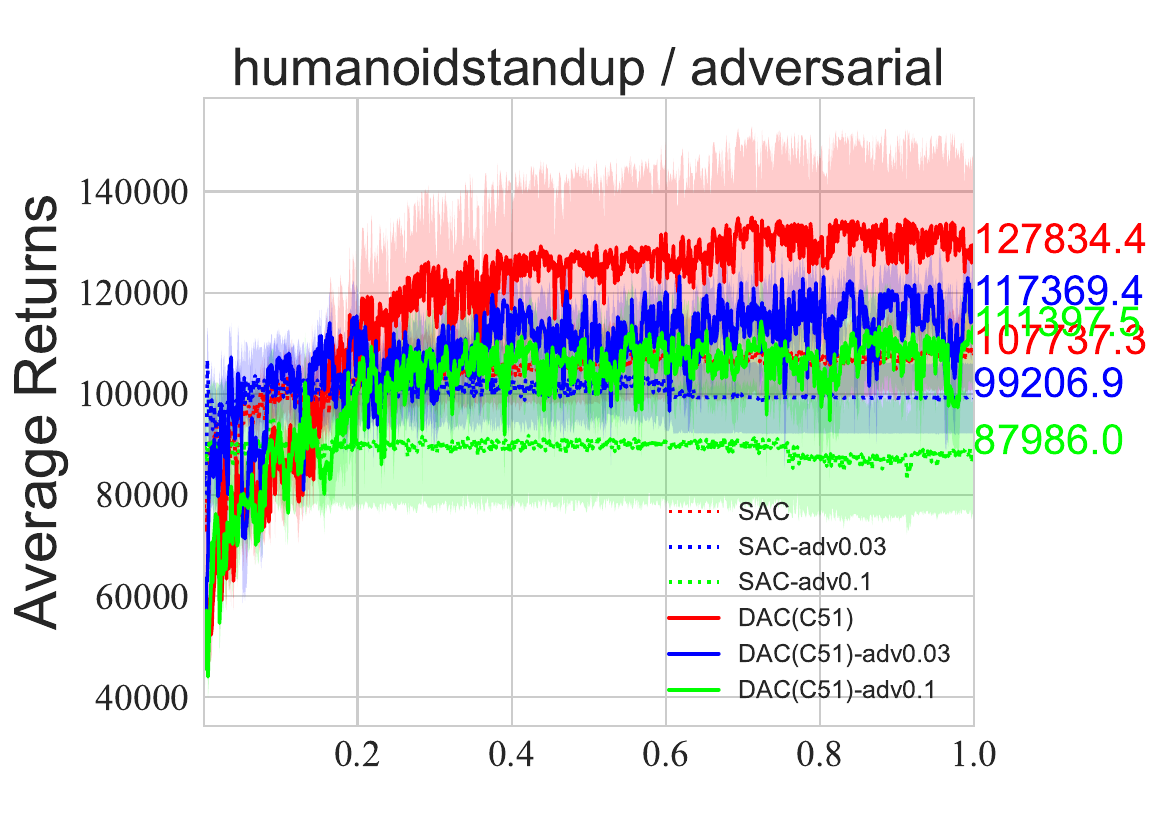}
	\end{subfigure}
	\begin{subfigure}[t]{0.45\textwidth}
		\centering
		\includegraphics[width=\textwidth,trim=10 0 0 10,clip]{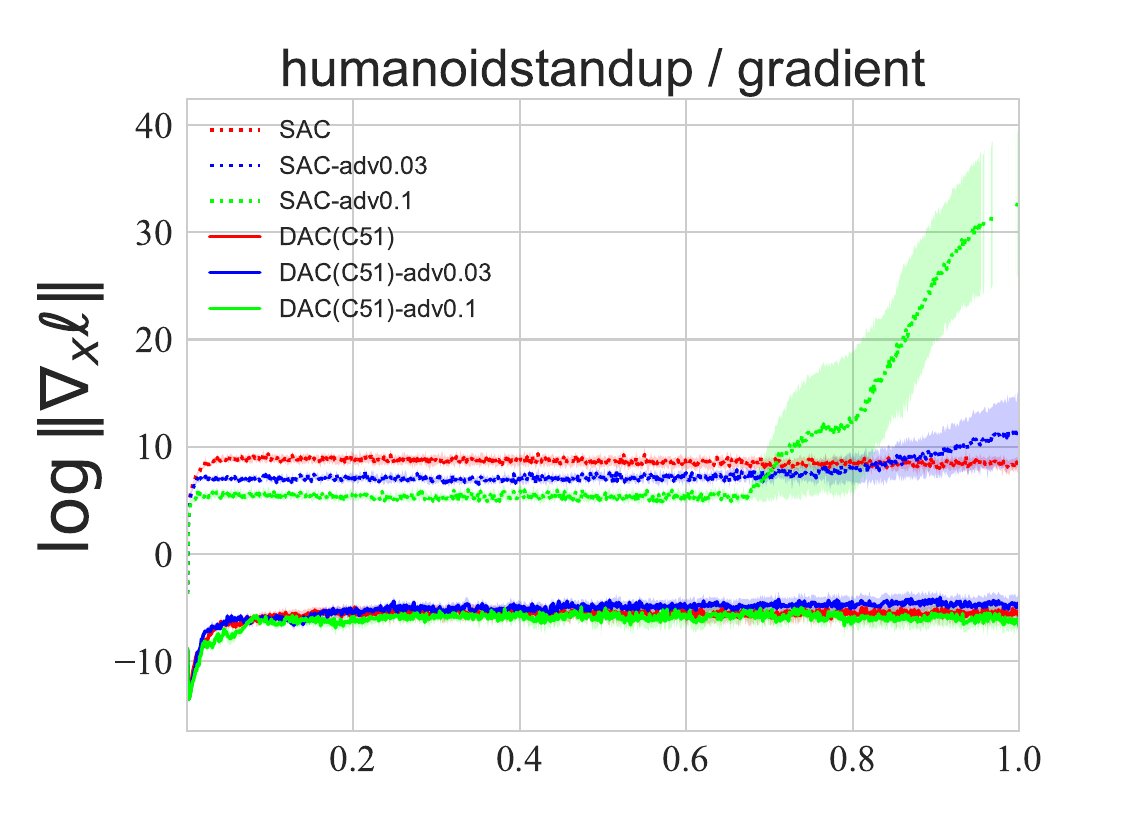}
	\end{subfigure}
	\caption{Average returns of SAC and DAC~(C51) against \textbf{adversarial} state observation noises in the training on Ant and Humanoidstandup under 5 runs. Gradient norms in the logarithm scale of AC and  DAC~(C51) in the adversarial setting. \textbf{advX} in the legend indicates random state observations with the perturbation size $\epsilon$ \textbf{X}. }
	\label{Figure_mujoco_grad}
\end{figure*}

\paragraph{Random and Adversarial State Noises.} We use Gaussian noise with different standard deviations to simulate random state noises, while for the adversarial state noise, we apply the most typical adversarial state perturbations proposed in \cite{huang2017adversarial,pattanaik2017robust}. For the choice of perturbation size, we followed \cite{zhang2020robust}, where the set of noises $B(s)$ is defined as an $\ell_{\infty}$ norm ball around $s$ with a radius $\epsilon$, given by $\ell_{\infty}B(s):=\left\{\hat{s}:\|s-\hat{s}\|_{\infty} \leq \epsilon\right\}$. We apply Projected Gradient Descent~(PGD) version in \cite{pattanaik2017robust}, with 3 fixed iterations while adjusting $\epsilon$ to control the perturbation strength. Due to the page limit, we defer similar results under more advanced MAD attack~\cite{zhang2020robust} in Appendix~\ref{appendix:attacks}.

\subsection{Results on Continuous Control Environments}

We compare SAC with DAC~(C51) on Ant and Humanoidstandup. Due to the space limit, we mainly present the algorithm performance in the \textbf{adversarial} setting. Figure~\ref{Figure_mujoco_grad} suggests that distributional RL algorithms, i.e., DAC~(C51), are less sensitive to their expectation-based counterparts, i.e., SAC, according to learning curves of average returns on Ant and Humanoidstandup. More importantly, Figure~\ref{Figure_mujoco_grad} demonstrates that DAC~(C51) enjoys smaller gradient norms compared with SAC, and SAC with a larger perturbation size is prone to unstable training with much larger gradient magnitudes. In particular, On Humanoidstandup, SAC converges undesirably with adv0.01 (green line), but its gradient norm diverges (even infinity in the very last phase). By contrast, DSAC~(C51) has a lower level gradient norms, which is less likely to suffer from divergence. This result corroborates with theoretical analysis in Section~\ref{sec:lipschitz} that exploding gradients are prone to divergence when exposed to state noises.

\begin{wraptable}[12]{r}{0.5\textwidth}
	\centering
	\scalebox{0.7}{
		\begin{tabular}{c|cccc}
			\toprule[1pt]
			\textbf{Robustness($\%$)}&\textbf{Adversarial}&\textbf{$\epsilon$=0.02}&\bf $\epsilon$=0.03& \bf $\epsilon$=0.1 \\
			\hline
			\multirow{2}*{Ant}&SAC&$\approx$ 0&$\approx$ 0&$\approx$ 0\\
			~&DAC~(C51)& \bf 74.0&\bf 48.3&\bf 20.5\\
			\hline
			\textbf{Robustness($\%$)}&\textbf{Adversarial}&\textbf{$\epsilon$=0.03}& \bf $\epsilon$=0.1&\\
			\hline
			\multirow{2}*{Humanoidstandup}&SAC&\bf 92.1&81.7&\\
			~&DAC~(C51)&91.8&\bf 87.1&\\
			\hline
			\bottomrule[1pt]
		\end{tabular}
	}
	\caption{Robustness ratio of algorithms under \textbf{adversarial} state observations with different $\epsilon$ on Ant and Humanoidstandup.}
	\label{table:robustness_mujoco}
\end{wraptable}

\begin{figure*}[t!]
	\centering
	\begin{subfigure}[t]{0.44\textwidth}
		\centering
		\includegraphics[width=\textwidth,trim=0 0 0 10,clip]{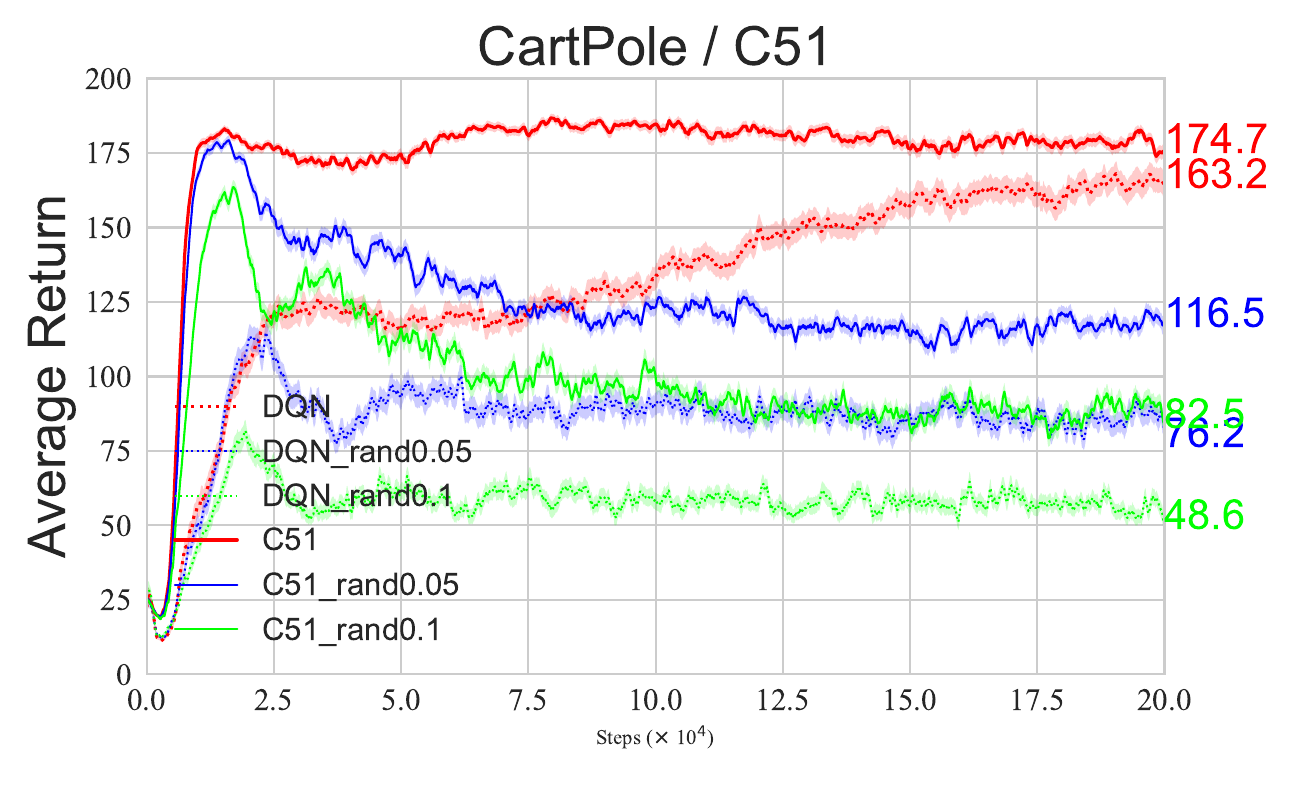}
	\end{subfigure}
	\begin{subfigure}[t]{0.44\textwidth}
		\centering
		\includegraphics[width=\textwidth,trim=0 0 0 10,clip]{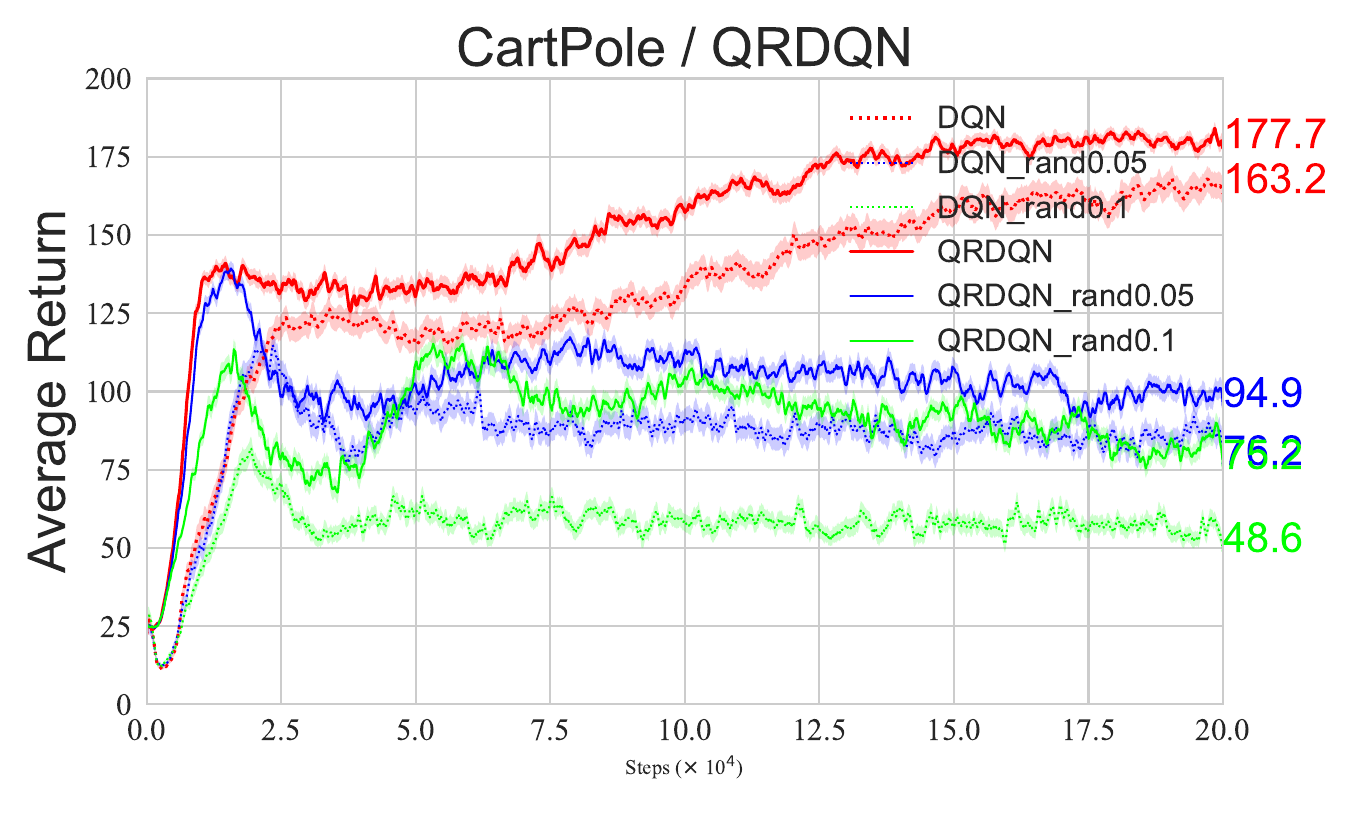}
	\end{subfigure}
	\begin{subfigure}[t]{0.45\textwidth}
		\centering
		\includegraphics[width=\textwidth,trim=0 0 0 10,clip]{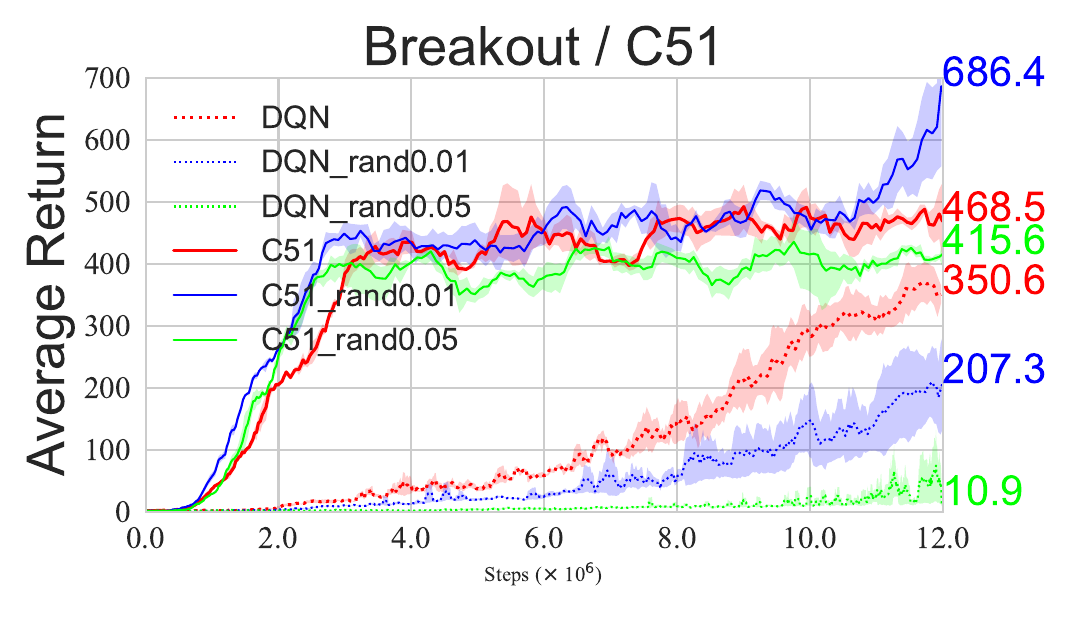}
	\end{subfigure}
	\begin{subfigure}[t]{0.44\textwidth}
		\centering
		\includegraphics[width=\textwidth,trim=10 0 10 10,clip]{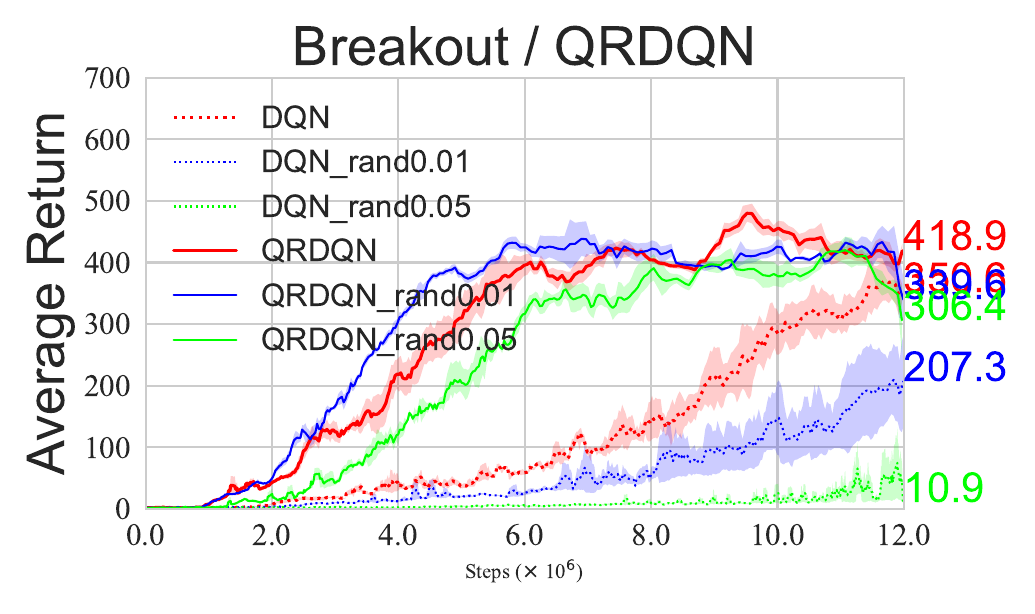}
	\end{subfigure}
	\caption{Average returns of DQN, C51 and QRDQN against \textbf{random} state observation noises on CartPole and Breakout. \textbf{randX} in the legend indicates random state observations with the standard deviation \textbf{X}.}
	\label{Figure_rand}
\end{figure*}

A quantitative result is also shown in Table~\ref{table:robustness_mujoco}, where distributional RL algorithms tend to maintain a higher robustness ratio as opposed to their expectation-based RL versions. We also note that the training robustness of distributional RL algorithms may not be significant if the perturbation size is slightly small, e.g., on Humanoidstandup. However, if we carefully vary perturbation sizes in a proper range, we can easily observe the robustness advantage of distributional RL against adversarial noises, e.g., on Ant. We also investigate the training robustness of more distributional RL algorithms over more games. Thus, we evaluate the sensitivity of D4PG~\cite{barth2018distributed} against adversarial noises on Halfcheetah, which can be viewed as the distributional version of DDPG. As suggested in Figure~\ref{Figure_mujoco_grad_newgame} in Appendix~\ref{appendix:D4PG}, the distributional RL algorithm D4PG is much less vulnerable than it expectation-based RL counterpart DDPG against adversarial noises.

\subsection{Results on Classical Control and Atari Games}

\begin{wraptable}[14]{r}{0.5\textwidth}
	\centering
	\scalebox{0.7}{
		\begin{tabular}{c|cccc}
			\toprule[1pt]
			\textbf{Robustness($\%$)}&\textbf{Random}&\textbf{std=0.05}&\bf std=0.1&\\
			\hline
			\multirow{3}*{CartPole}&DQN&44.2&28.6&\\
			~&QRDQN&54.5&43.4&\\
			~&C51&\bf 67.0&\bf 47.3&\\
			\hline
			\textbf{Robustness($\%$)}&\textbf{Random}&\textbf{std=0.01}&\bf std=0.05&\\
			\hline
			\multirow{3}*{Breakout}&DQN&59.1&$\approx 0$&\\
			~&QRDQN&81.1&73.1&\\
			~&C51&\bf 146.5&\bf 88.7&\\
			\hline
			\bottomrule[1pt]
		\end{tabular}
	}
	\caption{Robustness ratio of three algorithms under \textbf{random} state observations with different standard deviations~(std) on CartPole and Breakout.}
	\label{table:robustness_random}
\end{wraptable}

\paragraph{Results under Random State Noises.} We investigate the training robustness of DQN, C51 and QRDQN on classical control environments and typical Atari games, against the random noisy state observations. Gaussian state noises are continuously injected in the while training process of RL algorithms, while the agent encounters noisy current state observations while conducting the TD learning. Due to the space limit, here we mainly present learning curves of algorithms on CartPole and Breakout. As shown in Figure~\ref{Figure_rand}, both C51 and QRDQN achieve similar performance to DQN after the training \textit{without any random state noises}. However, when we start to inject random state noises with different noise sizes during the training process, their learning curves show different sensitivity and robustness. Both C51 and QRDQN are more robust against the random state noises than DQN, with the less interference for the training under the same random noises. Remarkably, in Breakout the performance of both C51 and QRDQN~(solid lines) only slightly decreases, while DQN~(dashed lines) degrades dramatically and even diverges when the standard deviation is 0.05. This significant difference provides a strong empirical evidence to verify the robustness advantage of distributional RL algorithms.

A detailed comparison is summarized in Table~\ref{table:robustness_random}. It turns out that  the training robustness of both QRDQN and C51 surpass DQN significantly. Note that the robustness ratio for C51 under std=0.01 noises is 146.5$\%$, which is above 100$\%$. This can be explained as a proper randomness added in the training might be beneficial to exploration, yielding better generalization of algorithms.

\begin{wraptable}[13]{r}{0.5\textwidth}
	\centering
	\scalebox{0.7}{
		\begin{tabular}{c|cccc}
			\toprule[1pt]
			\textbf{Robustness($\%$)}&\textbf{Adversarial}&\textbf{$\epsilon$=0.05}&\bf $\epsilon$=0.1&\\
			\hline
			\multirow{3}*{CartPole}&DQN&34.8&18.6&\\
			~&QRDQN&26.0&24.8&\\
			~&C51&\bf 75.6&\bf 70.6&\\
			\hline
			\textbf{Robustness($\%$)}&\textbf{Adversarial}&\textbf{$\epsilon$=0.0005}&\bf $\epsilon$=0.001&\\
			\hline
			\multirow{3}*{Breakout}&DQN&29.8&$\approx 0$&\\
			~&QRDQN&\bf 107.1 &\bf 132.6&\\
			~&C51&61.0&6.3&\\
			\hline
			\bottomrule[1pt]
		\end{tabular}
	}
	\caption{Robustness ratio of three algorithms under \textbf{adversarial} state observations with different perturbation sizes $\epsilon$ on CartPole and Breakout.}
	\label{table:robustness_adv}
\end{wraptable}

\paragraph{Results under Adversarial State Noises.} Next, we probe the training robustness of DQN, QRDQN and C51 in the setting where the agent encounters the \textit{adversarial} state observations in the current state in the function approximation case. Figure~\ref{Figure_adv} presents the learning curves of algorithms on CartPole and Breakout against noisy states under different adversarial perturbation sizes $\epsilon$.

\begin{figure*}[t!]
	\centering
	\begin{subfigure}[t]{0.44\textwidth}
		\centering
		\includegraphics[width=\textwidth,trim=0 0 0 10,clip]{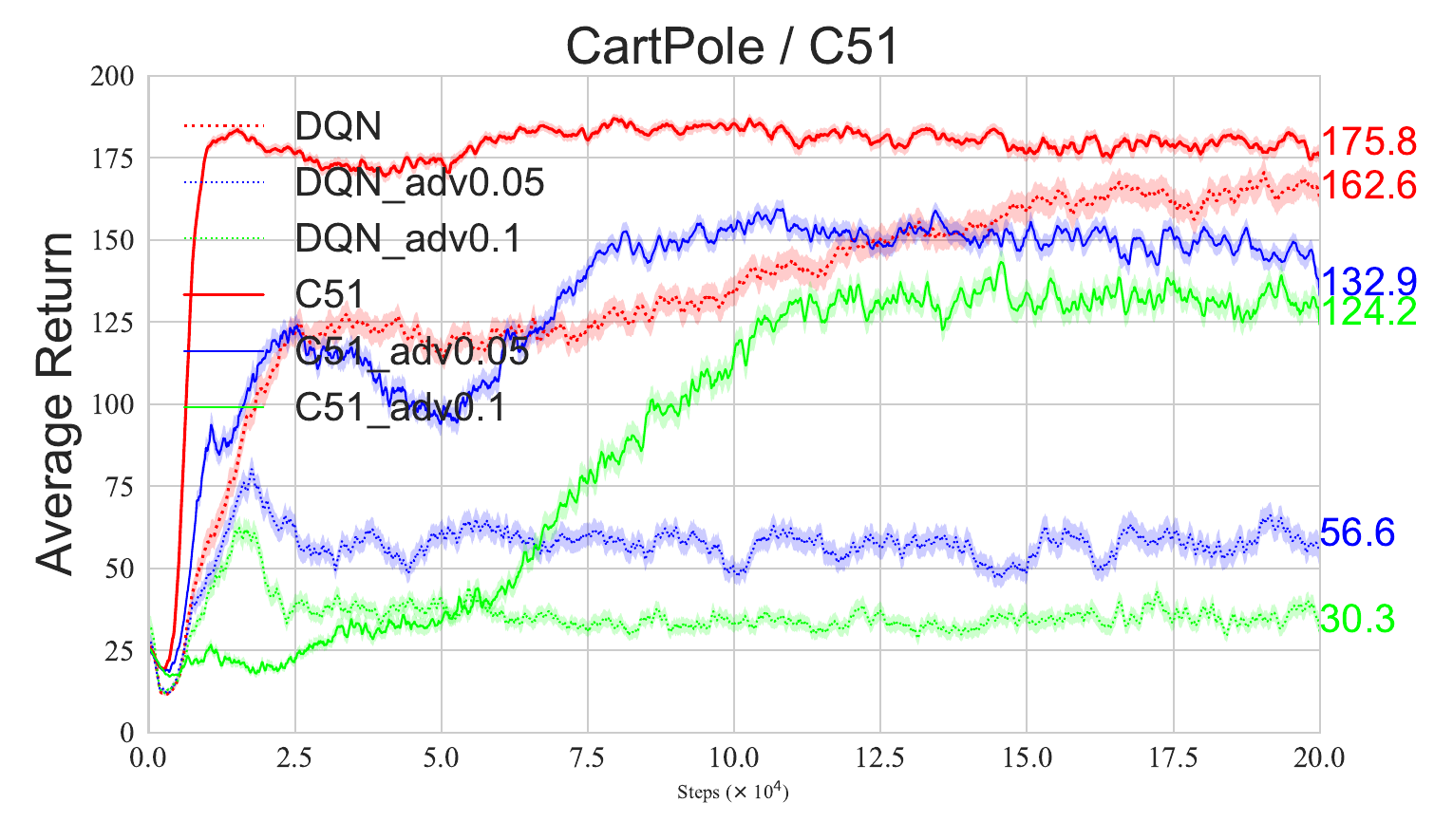}
	\end{subfigure}
	\begin{subfigure}[t]{0.44\textwidth}
		\centering
		\includegraphics[width=\textwidth,trim=0 0 0 10,clip]{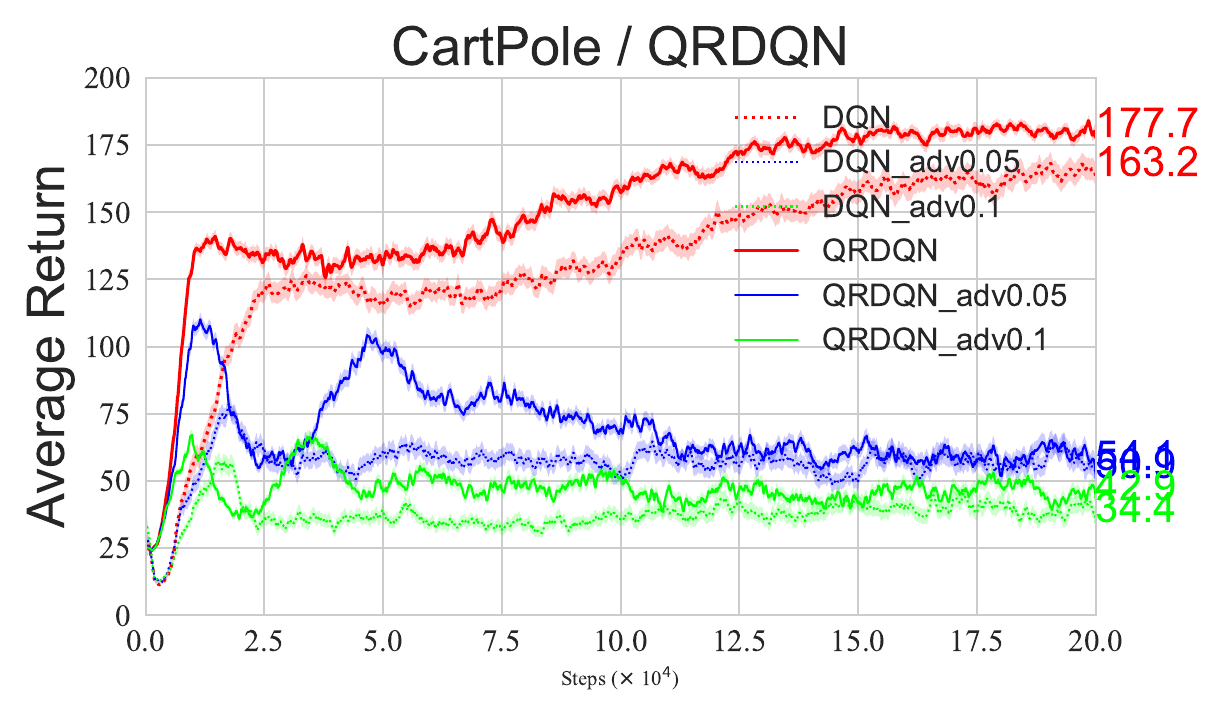}
	\end{subfigure}
	\begin{subfigure}[t]{0.46\textwidth}
		\centering
		\includegraphics[width=\textwidth,trim=0 0 0 10,clip]{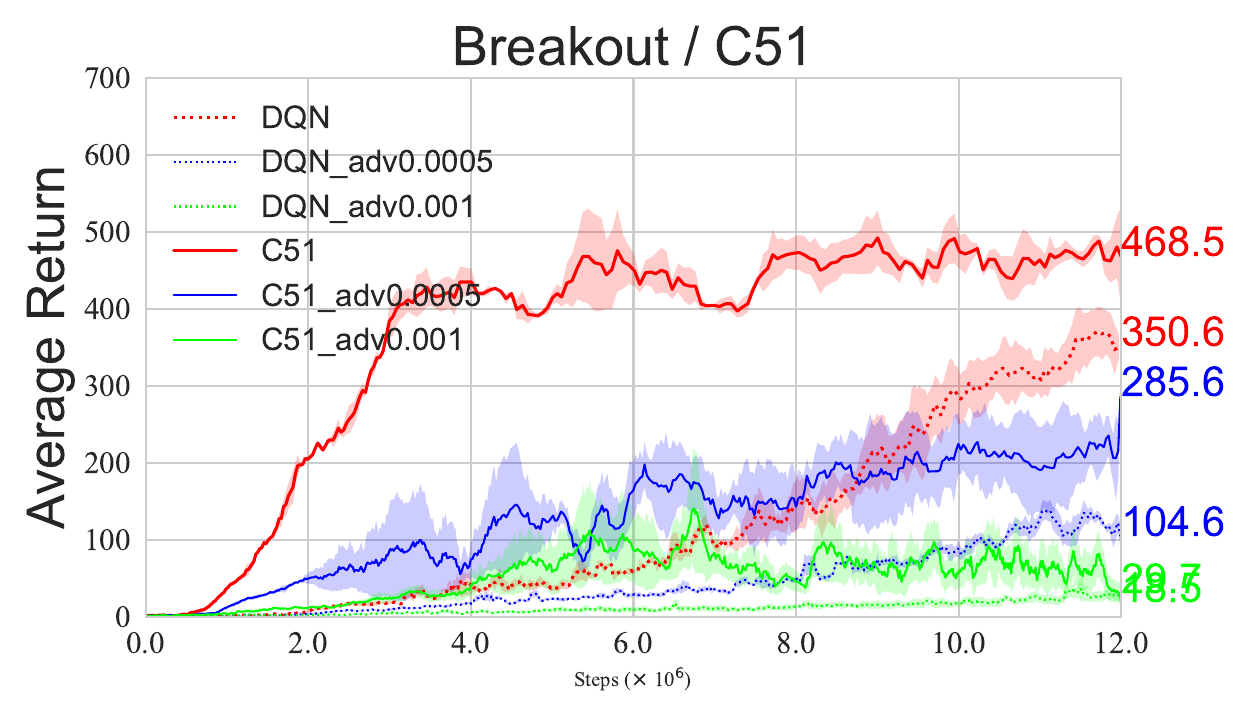}
	\end{subfigure}
	\begin{subfigure}[t]{0.41\textwidth}
		\centering
		\includegraphics[width=\textwidth,trim=10 20 10 10,clip]{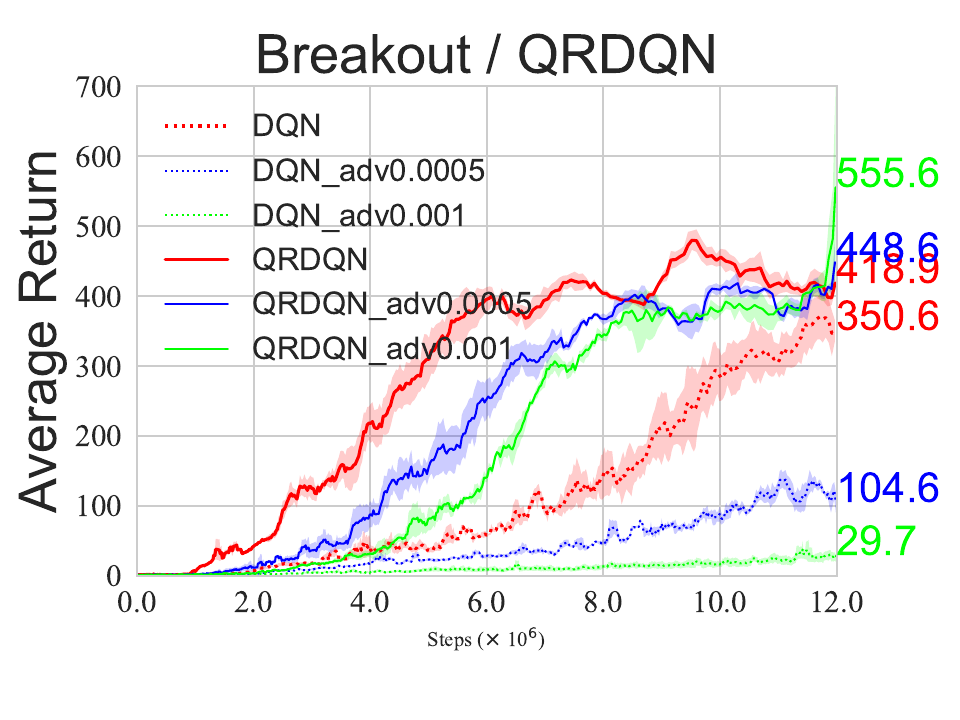}
	\end{subfigure}
	\caption{Average returns of DQN, C51 and QRDQN against \textbf{adversarial} state observation noises across four games. \textbf{advX} in the legend indicates random state observations with the perturbation size $\epsilon$ \textbf{X}. }
	\label{Figure_adv}
\end{figure*}

It turns out that results under the adversarial state observations are similar to those in the random noises case. Specifically, all algorithms tend to degrade when getting exposed to adversarial state observations, and even are more likely to diverge. However, a key observation is that \textit{distributional RL algorithms, especially QRDQN, are capable of obtaining desirable performance even when DQN diverges}. For instance, in Breakout DQN~(dotted green line) in Figure~\ref{Figure_adv} under the adversarial perturbation with $\epsilon=0.001$ leads to divergence, while QRDQN~(solid green lines) still maintains a desirable performance. The quantitative robustness ratio comparison is also provided in Table~\ref{table:robustness_adv}. It suggests that the adversarial robustness of C51 is superior to DQN and QRDQN in CartPole, while QRDQN is remarkably less sensitive to adversarial noises than both DQN and C51 in Breakout.

\begin{wraptable}[11]{r}{0.5\textwidth}
	\centering
	\scalebox{0.7}{
		\begin{tabular}{c|cccc}
			\toprule[1pt]
			\textbf{Robustness($\%$)}&\textbf{Algorithms}&\textbf{ std=0.0125}&\bf $\epsilon$=0.1&\\
			\hline
			\multirow{2}*{MountainCar}&DQN&32.4&32.5&\\
			~&QRDQN&\bf 79.0&\bf 44.7&\\
			\hline
			\textbf{Robustness($\%$)}&\textbf{Algorithms}&\textbf{std=0.05}&\bf $\epsilon$=0.005&\\
			\hline
			\multirow{2}*{Qbert}&DQN&10.8&6.3&\\
			~&QRDQN&\bf 34.5&\bf 32.9&\\
			\hline
			\bottomrule[1pt]
		\end{tabular}
	}
	\caption{Robustness ratio of DQN and QRDQN  under random and adversarial state noises  on MountainCar and Qbert.}
	\label{table:robustness_more}
\end{wraptable}

\paragraph{Results on MountainCar and Qbert.} Due to the space limit, we mainly summarize the robustness ratio of algorithms on MountainCar and Qbert in Table~\ref{table:robustness_more}. It turns out that the training robustness of QRDQN is significantly advantageous over DQN on both MountainCar and Qbert environments across two types of state noises, which also corroborates the robustness advantage of distributional RL algorithms over their  expectation-based RL counterpart.

\paragraph{Sensitivity Analysis of Different Perturbed States and Influence Function Analysis.} We also conduct experiments to verify the sensitivity analysis results~(Theorems~\ref{theorem:TD} and \ref{theorem:IF}) in Appendix~\ref{appendix:TD} about more TD convergence conditions and influence function in Appendix~\ref{appendix:IF}. These empirical evidence also coincides with our theoretical results.

\section{Discussion and Conclusion}\label{sec:discussion}

The robustness advantage analysis is based on the categorical distributional RL with categorical parameterization and the choice of KL divergence between current and target return distributions. However, it would be more convincing if we can still have such an analytical conclusion under Wasserstein distance. Moreover, we attribute the robustness advantage of distributional RL algorithms into the unbounded gradient norms regarding state features, but other factors, e.g., representation ability, may also contribute to the training robustness. We leave the exploration towards this direction as future works.

In this paper, we explore the training robustness of distributional RL against both random and adversarial noisy state observations. Based on the convergence proof of distributional RL in the SN-MDP, we further reveal the stable gradient behavior of distributional RL loss as opposed to classical RL, accounting for its less vulnerability. Empirical observations coincides with our theoretical results.

\section*{Acknowledgements}

We would like to thank the anonymous reviewers for great feedback on the paper. Dr. Kong was supported by the Natural Sciences and Engineering Research Council of Canada (NSERC), the University of Alberta/Huawei Joint Innovation Collaboration, Huawei Technologies Canada Co., Ltd., and Canada Research Chair in Statistical Learning, the Alberta Machine Intelligence Institute (Amii), and Canada CIFAR AI Chair (CCAI).   Yingnan Zhao and Ke Sun were supported by the State Scholarship Fund from China Scholarship Council (No:202006120405 and No:202006010082). 

\bibliography{dRL}
\bibliographystyle{splncs04}

\clearpage
\appendix

\section{Theorem~\ref{theorem:SNMDP} with proof}\label{appendix:theorem:SNMDP}

\begin{theorem}\label{theorem:SNMDP}(Convergence and Contraction of Bellman operators in the SN-MDP) Given a policy $\pi$, define the Bellman operator $\mathcal{T}:\mathbb{R}^{|S|}\rightarrow \mathbb{R}^{|S|}$ under random and adversarial states noises by $\mathcal{T}^{\pi}_r$ and $\mathcal{T}^{\pi}_a$, respectively. Denote a ``merged'' policy $\pi^\prime$ where $\pi^{\prime}(a|s)=\sum_{v(s)} N(v(s)|s) \pi(a|v(s))$ and $\mathbf{S}(\pi)$ is a policy set given $\pi$. Then we have:
	
	(1) $\mathcal{T}^{\pi}_r$ is a contraction operator and can converge to $V_{\pi^\prime}$, i.e., $\mathcal{T}^{\pi}_r \tilde{V}_{\pi \circ N} = \tilde{V}_{\pi \circ N}=V_{\pi^\prime}$, where multiple policies $\pi_{r} \in \mathbf{S}(\pi)$ might exist with $	\sum_{v(s)} N(v(s)|s) \pi_r(a|v(s))= \pi^{\prime}(a|s)$.

	(2) $\mathcal{T}^{\pi}_a$ is a contraction with the convergence satisfying $\mathcal{T}^{\pi}_a \tilde{V}_{\pi \circ N^*} = \min_{N} \tilde{V}_{\pi \circ N}=V_{\pi \circ N^*}$, where $N^*$ is the optimal adversarial noise strategy. If the optimal policy $\pi_a$ exists, it satisfies $\pi_a(a|v^*(s))=\pi(a|s)$ for each $s$ and $a$, where $v^*(s)$ is the adversarial noisy state manipulated by $N^*(\cdot|s)$.
\end{theorem}

\begin{proof}
	Our proof is partly based on Theorem 1 and 2 in \cite{zhang2020robust}, but adds more analysis on the converged policy especially under the random noisy states setting. The most important insight in the following proof is that the noise transition can be merged into the agent's policy, resulting in a new ``merged'' policy $\pi^{\prime}$.
	
	\paragraph{Proof of (1)} Firstly, as the Bellman Equation under the random noisy states is exactly the general form in Eq.~\ref{eq:BE_SNMDP}, it automatically satisfies that $\mathcal{T}^{\pi}_r \tilde{V}_{\pi \circ N} = \tilde{V}_{\pi \circ N}$ when it converges. As for the proof of contraction, based on our insight about the new ``merged'' policy $\pi^{\prime}$ where $\pi^{\prime}(a|s) = \sum_{v(s)} N(v(s)|s) \pi(a|v(s))$, we can rewrite our Bellman Operator as:
	\begin{equation}\label{eq:BE_SNMDP_rewrite} \begin{aligned}
			\mathcal{T}^{\pi}_r \tilde{V}_{\pi \circ N}(s) &= \sum_{a} \pi^{\prime}(a|s) \sum_{s^\prime}p(s^{\prime}|s,a) \left[R(s,a,s^{\prime}) + \gamma \tilde{V}_{\pi \circ N}(s^\prime) \right]\\
			&=\boldsymbol{R}(s) + \gamma \sum_{s^\prime}  \boldsymbol{P}^{\prime}_{s,s^\prime} \tilde{V}_{\pi \circ N}(s^\prime)
	\end{aligned}\end{equation}
	where $\boldsymbol{R}(s)=\sum_{a} \pi^{\prime}(a|s) \sum_{s^\prime}p(s^{\prime}|s,a) R(s,a,s^{\prime})$, and $\boldsymbol{P}^{\prime}_{s,s^\prime}=\sum_{a} \pi^{\prime}(a|s) p(s^{\prime}|s,a)$ determined by the ``merged'' policy $\pi^\prime$. Then for two different value function $\tilde{V}_{\pi \circ N}^1$ and $\tilde{V}_{\pi \circ N}^2$ we have:
	\begin{equation}\begin{aligned}
			\Vert \mathcal{T}^{\pi}_r \tilde{V}_{\pi \circ N}^1 - \mathcal{T}^{\pi}_r \tilde{V}_{\pi \circ N}^2 \Vert_{\infty} &= \max_s |\gamma \sum_{s^\prime}  \boldsymbol{P}^{\prime}_{s,s^\prime} \tilde{V}^1_{\pi \circ N}(s^\prime) - \gamma \sum_{s^\prime}  \boldsymbol{P}^{\prime}_{s,s^\prime} \tilde{V}^2_{\pi \circ N}(s^\prime)|\\	
			& \leq \gamma \max_s \sum_{s^\prime}  \boldsymbol{P}^{\prime}_{s,s^\prime}  | \tilde{V}^1_{\pi \circ N}(s^\prime) - \tilde{V}^2_{\pi \circ N}(s^\prime)| \\
			& \leq \gamma \max_s \sum_{s^\prime}  \boldsymbol{P}^{\prime}_{s,s^\prime} \max_{s^\prime} | \tilde{V}^1_{\pi \circ N}(s^\prime) - \tilde{V}^2_{\pi \circ N}(s^\prime)| \\
			& = \gamma \max_s \sum_{s^\prime}  \boldsymbol{P}^{\prime}_{s,s^\prime} \Vert \tilde{V}^1_{\pi \circ N} - \tilde{V}^2_{\pi \circ N}\Vert_{\infty} \\
			& = \gamma \Vert \tilde{V}^1_{\pi \circ N} - \tilde{V}^2_{\pi \circ N}\Vert_{\infty}
	\end{aligned}\end{equation}
	Then according to the Banach fixed-point theorem, since $\gamma \in (0,1)$, $\tilde{V}_{\pi \circ N}$ converges to a unique fixed-point $V_{\pi^{\prime}}$. However, even though the obtained policy $\pi^\prime$ satisfies that $\pi^\prime(a|s)=\sum_{v(s)} N(v(s)|s) \pi(a|v(s))$ for each $s, a$, these equations can not necessarily guarantee a unique $\pi$ especially when these equations behind this condition are underdetermined. In such scenario, multiple policies $\pi_r$ will exist as long as they satisfy the equations above.
	
	\paragraph{Proof of (2)}
	
	Firstly, based on Theorem 1~\cite{zhang2020robust} that shows an optimal policy does not always exist, we assume that an optimal policy exists in the adversarial noisy state setting for the convenience of following analysis. Based on this assumption, we need to derive the explicit value function under the adversary. Inspired by \cite{zhang2020robust}, the proof insight is that the behavior of optimal adversary can be also viewed as finding another optimal policy, yielding a zero-sum two player game. Specifically, in the SN-MDP setting, the adversary selects an action $\hat{a}\in \mathcal{S}$ satisfying $\hat{a}=v(s)$, attempting to maximize its state-action value function $\tilde{Q}_{\pi_a}(s,\hat{a})$. Then the adversary's value function $\hat{V}_{\pi_a}(s)$ can be formulated as:	
	\begin{equation} \begin{aligned}
			\hat{V}_{\pi_a}(s) & = \max_{\hat{a}} \hat{Q}_{\pi_a}(s,\hat{a}) \\ 
			& = \max_{\hat{a}} \sum_{s^\prime}\hat{p}(s^\prime|s,\hat{a})(\hat{R}(s,\hat{a},s^\prime)+\gamma \hat{V}_{\pi_{a}}(s^\prime))\\
			& = \max_{v^(s)} \sum_{s^\prime} \sum_{a}\pi(a|v(s))p(s^\prime|s,a)(-R(s,a,s^\prime) + \gamma \hat{V}_{\pi_{a}}(s^\prime)) \\
	\end{aligned}\end{equation}
	where $\hat{p}(s^\prime|s,\hat{a})$ is the transition dynamics of the adversary, satisfying $\hat{p}(s^\prime|s,\hat{a})=\sum_{a}\pi(a|v(s))p(s^\prime|s,a)$ from the perspective of the agent. $\hat{R}(s,\hat{a},s^\prime)$ is the adversary's reward function while taking action $\hat{a}$, which is the opposite number of $R(s,a,s^\prime)$ given the action $a$. In addition, since both the adversary and agent can serve as a zero-sum two-player game, it indicates that $\tilde{V}_{\pi_a}(s)=-\hat{V}_{\pi_a}(s)$ for the agent's value function $\tilde{V}_{\pi_a}$ in the adversary setting. Then we rearrange the equation above as follows:
	\begin{equation} \begin{aligned}
			\tilde{V}_{\pi_a}(s) &  = - \hat{V}_{\pi_a}(s) \\
			& =  - \min_{N(\cdot|s)} \sum_{s^\prime} \sum_{a}\pi^{\prime}(a|s)p(s^\prime|s,a)(-R(s,a,s^\prime)  - \gamma \tilde{V}_{\pi_{a}}(s^\prime)) \\
			& =  \min_{v(s)} \sum_{s^\prime} \sum_{a}\pi^{\prime}(a|s)p(s^\prime|s,a)(R(s,a,s^\prime)   + \gamma \tilde{V}_{\pi_{a}}(s^\prime)) \\
			& =  \min_{N(\cdot|s)} \sum_{s^\prime} \sum_{a}\pi^{\prime}(a|s)p(s^\prime|s,a)  (r_{t+1} + \gamma  \min_{N} \mathbb{E}_{\pi \circ N}\left[\sum_{k=0}^{\infty} r_{t+k+2}|s_{t+1}=s^\prime\right] ) \\
			& = \min_{N} \tilde{V}_{\pi \circ N}(s)
	\end{aligned}\end{equation}
	
	Note that we optimize over $N$, which means we consider $N(\cdot|s)$ for each state $s$. Further, we derive the contraction of the Bellman operator $\mathcal{T}^{\pi}_a$. We rewrite our Bellman Operator $\mathcal{T}^{\pi}_a$ as:
	\begin{equation} \begin{aligned}
			\mathcal{T}^{\pi}_a \tilde{V}_{\pi \circ N}(s)&= \min_{N} \tilde{V}_{\pi \circ N}(s) \\
			&=\min_{N} \boldsymbol{R}(s) + \gamma \sum_{s^\prime}  \boldsymbol{P}^{\prime}_{s,s^\prime} \tilde{V}_{\pi \circ N}(s^\prime)
	\end{aligned}\end{equation}
	We firstly assume $\mathcal{T}^{\pi}_a \tilde{V}_{\pi_a}^1(s) \geq \mathcal{T}^{\pi}_a \tilde{V}_{\pi_a}^2(s)$, then we have:
	\begin{equation}\begin{aligned}
			& \mathcal{T}^{\pi}_a \tilde{V}_{\pi \circ N}^1(s) - \mathcal{T}^{\pi}_a \tilde{V}_{\pi  \circ N}^2(s)\\ 
			& \leq \max_{N(\cdot|s)} \{\gamma \sum_{s^\prime}  \boldsymbol{P}^{\prime}_{s,s^\prime} \tilde{V}^1_{\pi \circ N}(s^\prime) - \gamma \sum_{s^\prime}  \boldsymbol{P}^{\prime}_{s,s^\prime} \tilde{V}^2_{\pi \circ N}(s^\prime)\}\\	
			& \leq \gamma \max_{N(\cdot|s)} \sum_{s^\prime}  \boldsymbol{P}^{\prime}_{s,s^\prime}  | \tilde{V}^1_{\pi \circ N}(s^\prime) - \tilde{V}^2_{\pi \circ N}(s^\prime)| \\
			& \leq \gamma \max_{N(\cdot|s)} \sum_{s^\prime}  \boldsymbol{P}^{\prime}_{s,s^\prime} \max_s | \tilde{V}^1_{\pi \circ N}(s^\prime) - \tilde{V}^2_{\pi \circ N}(s^\prime)| \\
			& = \gamma \max_{N(\cdot|s)} \sum_{s^\prime}  \boldsymbol{P}^{\prime}_{s,s^\prime} \Vert \tilde{V}^1_{\pi \circ N} - \tilde{V}^2_{\pi \circ N}\Vert_{\infty} \\
			& \leq \gamma \Vert \tilde{V}^1_{\pi \circ N} - \tilde{V}^2_{\pi \circ N}\Vert_{\infty}
	\end{aligned}\end{equation}
	where the first inequality holds as $\min_{x_1}f(x_1)-\min_{x_2}g(x_2)\leq \max_x (f(x)-g(x))$ and we extends this inequality into the Wasserstein distance in the proof of convergence of distributional RL setting in Appendix~\ref{appendix:theorem:SNMDP_dRL}. The last inequality holds since only $\boldsymbol{P}^{\prime}_{s,s^\prime}$ depends on $N(\cdot|s)$ while the infinity norm is a constant, which is independent with the current $N(\cdot|s)$. Similarly, the other scenario can be still proved. Thus, we have:
	\begin{equation}\begin{aligned}
			\Vert \mathcal{T}^{\pi}_a \tilde{V}_{\pi \circ N}^1 - \mathcal{T}^{\pi}_a \tilde{V}_{\pi  \circ N}^2 \Vert_{\infty} \leq \gamma \Vert \tilde{V}^1_{\pi \circ N} - \tilde{V}^2_{\pi \circ N}\Vert_{\infty}
	\end{aligned}\end{equation}
	Thus, we proved that $\mathcal{T}^{\pi}_a$ is still a contraction and converge to $\min_{N} \tilde{V}_{\pi \circ N}$. We denote it as $\tilde{V}_{\pi \circ N^*}$ In addition, based on the insight of the ``merged'' policy $\pi_a^{\prime}$, we have $\pi_a^{\prime}=\sum_{v(s)} N^*(v(s)|s) \pi(a|v(s))=\pi(a|v^*(s))$ where the deterministic state $v^*(s)$ is the adversarial noisy state from the state $s$.
	
\end{proof}

\section{Proof of Theorem~\ref{theorem:SNMDP_dRL}}\label{appendix:theorem:SNMDP_dRL}
\begin{proof}
	Firstly, we will provide the properties of Wassertein distance $d_p$ in Lemma~\ref{theorem:lemma} that we leverage in our following convergence proof. The $p$-Wasserstein metric $d_p$ is defined as 
	\begin{equation}\label{eq:wasserstein}\begin{aligned}
			d_p =\left(\int_{0}^{1}\left|F_{Z^*}^{-1}(\omega)-F_{Z_{\theta}}^{-1}(\omega)\right|^{p} d \omega\right)^{1 / p},
	\end{aligned}\end{equation}
	which minimizes the distance between the true return distribution $Z^*$ and the parametric distribution $Z_{\theta}$. $F^{-1}$ is the inverse cumulative distribution function of a random variable with the cumulative distribution function as $F$. 
	
	\begin{lemma}\label{theorem:lemma} (Properties of Wasserstein  Metric) We consider the distribution distance between the random variable $U$ and $V$. Denote $d_p$ as the Wasserstein distance between two distribution defined in Eq.~\ref{eq:wasserstein}. For any scalar $a$ and random variable $A$ independent of $U$ and $V$, the following relationships hold:
		\begin{equation}\begin{aligned}
				d_p(a U, a V) &\leq |a| d_p(U, V)\\
				d_p(A+U, A+V) &\leq d_p(U, V)\\
				d_p(AU, AV) &\leq \Vert A \Vert_{p} d_p(U, V)\\
		\end{aligned}\end{equation}
		Further, let $A_1, A_2, ...$ be a set of random variables describing the a partition of $\omega$, when the partition lemma holds:
		\begin{equation}\begin{aligned}
				d_p(U, V) &\leq \sum_{i} d_p(A_i U, A_i V).
		\end{aligned}\end{equation}
	\end{lemma}
	
	Then, the following contraction proof is in the maximal form of $d_p$ and we denote it as $\bar{d_p}$.
	\paragraph{Proof of (1)} This contraction proof is similar to the original one~\cite{bellemare2017distributional} in the distributional RL without state observation noises. The only difference lies in the new transition operator $\mathcal{P}^{\pi}_r$, but it dose not change the main proof process. For two different random variables $Z_N^1$ and $Z_N^2$ about returns, we have:
	\begin{equation}\begin{aligned}
			&\quad \bar{d_p}(\mathfrak{T}^{\pi}_r Z_N^1, \mathfrak{T}^{\pi}_r Z_N^2)\\
			&=\sup_{s,a} d_p(\mathfrak{T}^{\pi}_r Z_N^1(s,a), \mathfrak{T}^{\pi}_r Z_N^2(s,a))\\
			&=\sup_{s,a} d_p(R(s,a,S^\prime)+\gamma\mathcal{P}^{\pi}_r Z_N^1(s,a), R(s,a,S^\prime)+\gamma \mathcal{P}^{\pi}_r Z_N^2(s,a)) \\
			&\leq \gamma \sup_{s,a} d_p(\mathcal{P}^{\pi}_r Z_N^1(s,a), \mathcal{P}^{\pi}_r Z_N^2(s,a))\\
			&\leq \gamma \sup_{s,a} \sup_{s^{\prime},a^{\prime}} d_p(Z_N^1(s^{\prime},a^{\prime}), Z_N^2(s^{\prime},a^{\prime}))\\
			&= \gamma \sup_{s^{\prime},a^{\prime}} d_p(Z_1(s^{\prime},a^{\prime}), Z_2(s^{\prime},a^{\prime}))\\
			&=\gamma \sup_{s,a} d_p(Z_N^1(s,a), Z_N^2(s,a)) \\
			&=\gamma \bar{d_p}(Z_N^1, Z_N^2).
	\end{aligned}\end{equation}
	Thus, we conclude that $\mathfrak{T}_r^{\pi}: \mathcal{Z}\rightarrow \mathcal{Z}$ is a $\gamma$-contraction in $\bar{d_p}$.	
	
	\paragraph{Proof of (2)} 
	Firstly, we define the distributional Bellman optimality operator $\mathfrak{T}$ in MDP as
	\begin{equation}\begin{aligned}
			\mathfrak{T} Z(s, a) : \stackrel{D}{=} R\left(s, a, S^{\prime}\right)+\gamma Z(S^\prime, \pi_Z(S^\prime))	
	\end{aligned}\end{equation}
	where $S^\prime \sim P(\cdot | s, a)$ and $\pi_Z(S^\prime)=\arg\max_{a^\prime}\mathbb{E}\left[Z(S^\prime, a^\prime)\right]$. By contrast, in SN-MDP, our greedy adversarial rule $N^*(\cdot | s^\prime)$ is based on the greedy policy rule in distributional Bellman optimality operator, which attempts to find adversarial $N^*(\cdot | s^\prime)$ in order to minimize $\mathbb{E}\left[Z_N(s^\prime, a^\prime)\right]$, where $a^\prime\sim \pi(\cdot|V(s^\prime))$ and $V(s^\prime)\sim N(\cdot | s^\prime)$. As $N^*(\cdot | s^\prime)$ yields a deterministic state $s^*$, the agent always takes action based on $s^*$, which we denote as $A^*\sim \pi(\cdot|s^*)$. Therefore, we can obtain the state-action function $Q^\pi_{N^*}(s, a)$ under the adversary as
	\begin{equation}\begin{aligned}
			Q^\pi_{N^*}(s, a)&=\min_N \mathbb{E}\left[Z_N^\pi (s, a)\right]  =\mathbb{E}\left[Z^{\pi^*} (s, a)\right]
	\end{aligned}\end{equation}
	where $\pi^*(\cdot|s)=\pi(\cdot|s^*)$ for $\forall s$ that follows the adversarial policy $A^*$. Next, to derive the contractive property of $\mathfrak{T}^\pi_a$, we denote two state-action valued distributions as $Z^{1}_{N}(s,a)$ and $Z^{2}_{N}(s,a)$, with the quantile function as $F^{-1}_{1, a}$ and $F^{-1}_{2, a}$. Then we have:		
	\begin{equation}\begin{aligned}
			 \bar{d_p}(\mathfrak{T}^{\pi}_a Z^1_N, \mathfrak{T}^{\pi}_a Z^2_N)
			&=\sup_{s,a} d_p(\mathfrak{T}^{\pi}_a Z^1_N(s,a), \mathfrak{T}^{\pi}_a Z^2_N(s,a))\\
			&=\sup_{s,a} d_p(R(s,a,S^\prime) + \gamma \mathcal{P}^{\pi}_a Z^1_{N}(s,a), R(s,a,S^\prime) + \gamma \mathcal{P}^{\pi}_a Z^2_{N}(s,a)) \\
			&\leq \gamma \sup_{s,a} \sum_{s^\prime} P(s^\prime | s, a) d_p(Z^1_{N}(s^\prime, A^*_1), Z^2_{N}(s^\prime,A^*_2))\\
			& \leq \gamma \sup_{s^\prime}  d_p(Z^1_{N}(s^\prime, a^*_1), Z^2_{N}(s^\prime,a^*_2)) \\
			&\overset{(a)}{\leq} \gamma \sup_{s^\prime, a^*} d_p(Z^1_{N}(s^\prime,a^*), Z^2_{N}(s^\prime,a^*))\\
			&= \gamma \sup_{s, a} d_p(Z^1_{N}(s, a), Z^2_{N}(s,a))\\
			&= \gamma \bar{d_p}(Z^1_N, Z^2_N)
	\end{aligned}\end{equation}
	where we explain $(a)$ when $p=1$. Without loss of generality, we assume $F^{-1}_{1, a_1^*}(w) > F^{-1}_{2, a_2^*}(w)$. Since $\mathbb{E}\left[Z_N^1(s^\prime, A^*_1)\right] = \min_{a_*} \int_{0}^{1} F^{-1}_{1, a^*} (w) dw$, we have 
	\begin{equation}\begin{aligned}
			d_p(Z^1_{N}(s^\prime, a^*_1), Z^2_{N}(s^\prime,a^*_2)) & = \int_{0}^{1} (F^{-1}_{1, a^*_1}(w) - F^{-1}_{2, a^*_2}(w)) dw \\
			& = \min_{a^*} \int_{0}^{1} F^{-1}_{1, a^*}(w) dw - \min_{a^*} \int_{0}^{1} F^{-1}_{2, a^*}(w) dw \\
			& \leq \max_{a^*}  \int_{0}^{1} (F^{-1}_{1, a^*}(w) - F^{-1}_{2, a^*}(w)) dw
	\end{aligned}\end{equation}
	where the last inequality is because $\min_x f(x) - \min_x g(x) \leq f(x_2) - g(x_2) \leq \max_{x} (f(x) - g(x))$. Thus, we conclude $\mathfrak{T}_a^\pi$ is  a $\gamma$-contraction in $\bar{d_p}$ when $p=1$.
\end{proof}

\section{Proof of Theorems~\ref{theorem:dRL_bounded_linear} and \ref{theorem:dRL_bounded_nonlinear}}\label{appendix:lipschitz}

\begin{proof}
	Firstly, we prove Theorem~\ref{theorem:dRL_bounded_linear}. We show the derivation details of the Histogram distribution loss starting from KL divergence between $p$ and $q_{\mathbf{w}}$. $p_i$ is the cumulative probability increment of target distribution $\mathfrak{T} Z_{\mathbf{w}}$ within the $i$-th bin, and $q_{\mathbf{w}}$ corresponds to a (normalized) histogram, and has density values $\frac{f_i^{\mathbf{w}}(\mathbf{x}(s))}{w_i}$ per bin. Thus, we have:
	\begin{equation}\begin{aligned}
			\mathcal{L}(Z_{\mathbf{w}}, \mathfrak{T} Z_{\mathbf{w}})& =-\int_{a}^{b} p(y) \log q_{\mathbf{w}}(y) d y \\ 
			&=-\sum_{i=1}^{k} \int_{l_{i}}^{l_{i}+w_{i}} p(y) \log \frac{f_{i}^{\mathbf{w}}(\mathbf{x}(s))}{w_{i}} d y\\
			&=-\sum_{i=1}^{k} \log \frac{f_{i}^{\mathbf{w}}(\mathbf{x}(s))}{w_{i}} \underbrace{\left(F_{\mathfrak{T}Z_{\mathbf{w}}}\left(l_{i}+w_{i}\right)-F_{\mathfrak{T}Z_{\mathbf{w}}}\left(l_{i}\right)\right)}_{p_{i}}\\
			&\doteq -\sum_{i=1}^{k} p_{i} \log f_{i}^{\mathbf{w}}(\mathbf{x}(s))
	\end{aligned}\end{equation}	
	where the last line holds because the width parameter $w_i$ can be ignored for this minimization problem and thus the loss function is proportion to the final term.
	
	Next, we compute the gradient of the Histogram distributional loss in the linear approximation case.
	\begin{equation}\begin{aligned}
			\frac{\partial}{\partial \mathbf{x}(s)} \sum_{j=1}^{k} p_{j} \log f^{\mathbf{w}}_{j}(\mathbf{x}(s)) 
			&=\sum_{j=1}^{k} p_{j} \frac{1}{f^{\mathbf{w}}_j(\mathbf{x}(s))} \nabla f^{\mathbf{w}}_j(\mathbf{x}(s))\\
			&=\sum_{j=1}^{k} p_{j} \frac{1}{f^{\mathbf{w}}_j(\mathbf{x}(s))} f^{\mathbf{w}}_j(\mathbf{x}(s)) \sum_{i=1}^{k}\frac{\exp(\mathbf{x}(s)^\top {\mathbf{w}}_{i})}{\sum_{p=1}^{k} \exp(\mathbf{x}(s)^\top {\mathbf{w}}_{p})} ({\mathbf{w}}_j - {\mathbf{w}}_i)\\
			&=\sum_{j=1}^{k} p_{j} \sum_{i=1}^{k} f^{\mathbf{w}}_i(\mathbf{x}(s)) ({\mathbf{w}}_{j}-{\mathbf{w}}_i)\\
			&=\sum_{j=1}^{k} p_{j} {\mathbf{w}}_{j} - \sum_{i=1}^{k} f^{\mathbf{w}}_i(\mathbf{x}(s)) {\mathbf{w}}_{i}\\
			&=\sum_{i=1}^k (p_i - f^{\mathbf{w}}_i(\mathbf{x}(s))) {\mathbf{w}}_i
	\end{aligned}\end{equation}	
	
	Then, as we have $\Vert {\mathbf{w}}_{i} \Vert \leq l$ for $\forall i$, we bound the norm of its gradient
	\begin{equation}\begin{aligned}
			\Vert \frac{\partial}{\partial \mathbf{x}(s)} \sum_{j=1}^{k} p_{j} \log f^{\mathbf{w}}_{j}(\mathbf{x}(s)) \Vert &\le \sum_{i=1}^k \Vert (p_i - f^{\mathbf{w}}_i(\mathbf{x}(s))) {\mathbf{w}}_i \Vert\\
			&= \sum_{i=1}^k |p_i - f^{\mathbf{w}}_i(\mathbf{x}(s))| \Vert {\mathbf{w}}_i \Vert \\
			&\leq kl
	\end{aligned}\end{equation}	
	The last equality satisfies because $|p_i - f^{\mathbf{w}}_i(\mathbf{x}(s))|$ is less than 1 and even smaller. In summary, compared with the least squared loss in expectation-based RL, the histogram distributional loss in distributional RL has the bounded gradient norm regarding the state features $\mathbf{x}(s)$. This upper bound of gradient norm can mitigate the impact of the noises on state observations on the loss function, therefore yielding training robustness for distributional RL.
	
	Next, we prove the Proposition~\ref{theorem:dRL_bounded_nonlinear}. Its proof is similar to Proposition~\ref{theorem:dRL_bounded_linear}. Firstly, we know that $f_{i}^{\mathbf{w}, \theta}(\mathbf{x}(s))=\exp \left(\phi_{\mathbf{w}}(\mathbf{x}(s))^{\top} \theta_{i}\right) / \sum_{j=1}^{k} \exp \left(\phi_{\mathbf{w}}(\mathbf{x}(s))^{\top} \theta_{j}\right)$ and $\phi_{\mathbf{w}}(\cdot)$ is L-Lipschitz, i.e., $\Vert \phi_{\mathbf{w}}(x) - \phi_{\mathbf{w}}(y) \Vert \leq L \Vert x -y \Vert$. Then 
	
	\begin{equation}\begin{aligned}
			&\frac{\partial}{\partial \mathbf{x}(s)} \sum_{j=1}^{k} p_{j} \log f^{\mathbf{w}, \theta}_{j}(\mathbf{x}(s))\\
			&=\sum_{j=1}^{k} p_{j} \frac{1}{f^{\mathbf{w}, \theta}_j(\mathbf{x}(s))} f^{\mathbf{w}, \theta}_j(\mathbf{x}(s)) \sum_{i=1}^{k}\frac{\exp(\mathbf{x}(s)^\top {\mathbf{w}}_{i})}{\sum_{p=1}^{k} \exp(\mathbf{x}(s)^\top {\mathbf{w}}_{p})} \cdot  (\nabla_{\mathbf{x}} \phi_{\mathbf{w}}^\top \theta_{j}  - \nabla_{\mathbf{x}} \phi_{\mathbf{w}}^\top \theta_{i})\\
			&=\sum_{j=1}^{k} p_{j} \sum_{i=1}^{k} f^{\mathbf{w}, \theta}_i(\mathbf{x}(s))  (\nabla_{\mathbf{x}} \phi_{\mathbf{w}}^\top \theta_{j}  - \nabla_{\mathbf{x}} \phi_{\mathbf{w}}^\top \theta_{i})\\
			&=\sum_{j=1}^{k} p_{j} \nabla_{\mathbf{x}} \phi_{\mathbf{w}}^\top \theta_{j} - \sum_{i=1}^{k} f^{\mathbf{w}, \theta}_i(\mathbf{x}(s)) \nabla_{\mathbf{x}} \phi_{\mathbf{w}}^\top \theta_{i}\\
			&=\sum_{i=1}^k (p_i - f^{\mathbf{w}, \theta}_i(\mathbf{x}(s))) \nabla_{\mathbf{x}} \phi_{\mathbf{w}}^\top \theta_{i}
	\end{aligned}\end{equation}	
	
	Then, as we have $\Vert \theta_{i} \Vert \leq l$ for $\forall i$, we bound the norm of its gradient
	\begin{equation}\begin{aligned}
			\Vert \frac{\partial}{\partial \mathbf{x}(s)} \sum_{j=1}^{k} p_{j} \log f^{\mathbf{w}, \theta}_{j}(\mathbf{x}(s)) \Vert & \le \sum_{i=1}^k \Vert (p_i - f^{\mathbf{w}, \theta}_i(\mathbf{x}(s))) \nabla_{\mathbf{x}} \phi_{\mathbf{w}}^\top \theta_{i} \Vert\\
			&= \sum_{i=1}^k |p_i - f^{\mathbf{w}, \theta}_i(\mathbf{x}(s))| \Vert \nabla_{\mathbf{x}} \phi_{\mathbf{w}}^\top \theta_{i} \Vert \\
			&\leq kl L
	\end{aligned}\end{equation}	
	
	The last inequality holds because $| \phi_{\mathbf{w}}(x)^\top \theta_{i} -  \phi_{\mathbf{w}}(y)^\top \theta_{i} | \leq \Vert \phi_{\mathbf{w}}(x)  -  \phi_{\mathbf{w}}(y) \Vert \Vert \theta_{i} \Vert \leq l L \Vert x -y \Vert $. Thus the function $\phi_{\mathbf{w}}^\top \theta_{i}$ can be viewed as $Ll$-Lipschitz continuous, indicating that $\Vert \nabla_{\mathbf{x}} \phi_{\mathbf{w}}^\top \theta_{i} \Vert \leq lL$.
\end{proof}

\section{TD Convergence Under Noisy State Observations}\label{appendix:TD}

\begin{theorem}\label{theorem:TD}(Conditions for TD Convergence  under Noisy State Observations) Define $\mathbf{P}$ as the $|\mathcal{S}| \times |\mathcal{S}|$ matrix forming from $p(s^\prime|s)$, $\mathbf{D}$ as the $|\mathcal{S}| \times |\mathcal{S}|$ diagonal matrix with $\mu(s)$ on its diagonal, and $\mathbf{X}$ as the $|\mathcal{S}| \times d$ matrix with $\mathbf{x}(s)$ as its rows, and $\mathbf{E}$ is the $|\mathcal{S}| \times d$ perturbation matrix with each perturbation vector $\mathbf{e}(s)$ as its rows. The stepsizes $\alpha_t \in (0,1]$ satisfy $\sum_{t=0}^{\infty} \alpha_{t}<\infty$ and $\sum_{t=0}^{\infty} \alpha_{t}^2=0$. For noisy states, we consider the following three cases: \textbf{(i)} $\mathbf{e}(s)$ on current state features, i.e., $\mathbf{x}_t \leftarrow \mathbf{x}_t+\mathbf{e}_t$, \textbf{(ii)} $\mathbf{e}(s^\prime)$ on next state features, i.e., $\mathbf{x}_{t+1} \leftarrow \mathbf{x}_{t+1}+\mathbf{e}_{t+1}$, \textbf{(iii)} the same $\mathbf{e}$ on both state features. We can attain that $\mathbf{w}_t$ converges to TD fixed point if the following conditions are satisfied, respectively.
	
	\textbf{Case (i):} both $\mathbf{A}$ and $(\mathbf{X}+\mathbf{E})^\top \mathbf{D} \mathbf{P} \mathbf{E}$ are positive definite. \textbf{Case (ii):} both $\mathbf{A}$ and $-\mathbf{X}^\top\mathbf{D} \mathbf{P} \mathbf{E}$ are positive definite. \textbf{Case (iii):} $\mathbf{A}$ is positive definite.
	
\end{theorem} 

From the convergence conditions for the three cases in Theorem~\ref{theorem:TD}, it is clear that (iii) is the mildest. This is the same condition as that in the normal TD learning without noisy state observations. Note that the case (iii) can be viewed as the SN-MDP setting, whose convergence has been already rigorously analyzed in Section~\ref{Section:SNMDP}. In Section~\ref{Section:experiment}, our experiments demonstrate that both expectation-based and distribution RL are more likely to converge in case (iii) compared with case (i) and (ii).

In cases (i) and (ii), the positive definiteness of $\mathbf{X}^\top\mathbf{D} \mathbf{P} \mathbf{E} + \mathbf{E}^\top\mathbf{D} \mathbf{P} \mathbf{E}$ and $-\mathbf{X}^\top\mathbf{D} \mathbf{P} \mathbf{E}$ is crucial. We partition $(\mathbf{X}+\mathbf{E})^\top \mathbf{D} \mathbf{P} \mathbf{E}$ into $\mathbf{X}^\top\mathbf{D} \mathbf{P} \mathbf{E} + \mathbf{E}^\top\mathbf{D} \mathbf{P} \mathbf{E}$, where the first term has the opposite positive definiteness to $-\mathbf{X}^\top\mathbf{D} \mathbf{P} \mathbf{E}$, and the second term is positive definite~\cite{sutton2018reinforcement}. Based on these observations, we discuss the subtle convergence relationship in cases (i) and (ii):

\textbf{(1)} If  $-\mathbf{X}^\top\mathbf{D} \mathbf{P} \mathbf{E}$ is positive definite, which indicates that TD is convergent in case (ii), TD can still converge in case (i) \textbf{unless} the positive definiteness of  $\mathbf{E}^\top\mathbf{D} \mathbf{P} \mathbf{E}$ dominates in $\mathbf{X}^\top\mathbf{D} \mathbf{P} \mathbf{E} + \mathbf{E}^\top\mathbf{D} \mathbf{P} \mathbf{E}$.

\textbf{(2)} If  $-\mathbf{X}^\top\mathbf{D} \mathbf{P} \mathbf{E}$ is negative definite, TD is likely to diverge in case (ii). By contrast, TD will converge in case (i).

In summary, there exists a subtle trade-off of TD convergence in case (i) and (ii) if we approximately ignore the term $\mathbf{E}^\top\mathbf{D} \mathbf{P} \mathbf{E}$ in case (i). The key of it lies in the positive definiteness of the matrix $\mathbf{X}^\top\mathbf{D} \mathbf{P} \mathbf{E}$, which heavily depends on the task. In Section~\ref{Section:experiment}, we empirically verify that the convergence situations for current and next state observations are normally different. Which situation is superior is heavily dependent on the task.

\begin{proof} 
	To prove the convergence of TD under the noisy states, we use the results from \cite{borkar2000ode} that require the condition about stepsizes $\alpha_t$ holds: $\sum_{t=0}^{\infty} \alpha_{t}<\infty$ and $\sum_{t=0}^{\infty} \alpha_{t}^2=0$. Our part proof is directly established on \cite{sutton2018reinforcement}. Particularly, the positive definiteness of $\mathbf{A}$ will determine the TD convergence. For linear TD(0), in the continuing case with $\gamma < 1$, $\mathbf{A}$ can be re-written as:
	\begin{equation}\label{eq:TD_A}\begin{aligned}
			\mathbf{A}&=\sum_{s} \mu(s) \sum_{a} \pi(a|s) \sum_{r, s^{\prime}} p(r, s^{\prime}|s, a) \mathbf{x}_t\left(\mathbf{x}_t-\gamma \mathbf{x}_{t+1}\right)^{\top}\\
			&=\sum_{s} \mu(s) \sum_{s^{\prime}} p(s^{\prime}|s) \mathbf{x}_t\left(\mathbf{x}_t-\gamma \mathbf{x}_{t+1}\right)^{\top}\\
			&=\sum_{s} \mu(s) \mathbf{x}_t (\mathbf{x}_t - \gamma \sum_{s^{\prime}} p(s^{\prime}|s) \mathbf{x}_{t+1})^{\top}\\
			&=\mathbf{X}^\top \mathbf{D} \mathbf{X} - \mathbf{X}^\top \mathbf{D} \gamma \mathbf{P} \mathbf{X}\\
			&=\mathbf{X}^\top \mathbf{D} (\mathbf{I} - \gamma \mathbf{P}) \mathbf{X}		
	\end{aligned}\end{equation}	
	Then we use $\mathbf{A}_t$ to present the convergence matrix in the case (i) where the perturbation vector $\mathbf{e}_t$ is added onto the current state features, i.e., $\mathbf{x}_t \leftarrow \mathbf{x}_t+\mathbf{e}_t$, while we use $\mathbf{A}_{t+1}$ and $\mathbf{A}_{t,t+1}$to present the counterparts in the case (ii) and (iii), respectively. 	Based on Eq.~\ref{eq:TD_A}, in the case (iii), we have:
	\begin{equation}\label{eq:TD_A_tt+1}\begin{aligned}
			\mathbf{A}_{t,t+1}&=(\mathbf{X}+\mathbf{E})^\top \mathbf{D} (\mathbf{X}+\mathbf{E}) - (\mathbf{X}+\mathbf{E})^\top \mathbf{D} \gamma \mathbf{P} (\mathbf{X} + \mathbf{E})\\
			&=(\mathbf{X}+\mathbf{E})^\top \mathbf{D} (\mathbf{I} - \gamma \mathbf{P}) (\mathbf{X}+\mathbf{E})
	\end{aligned}\end{equation}
	From \cite{sutton2018reinforcement}, we know that the inner matrix $\mathbf{D} (\mathbf{I} - \gamma \mathbf{P})$ is the key to determine the positive definiteness of $\mathbf{A}$. If we assume that $\mathbf{A}$ is positive definite, which also indicates that $\mathbf{D} (\mathbf{I} - \gamma \mathbf{P})$ is positive definite equivalently. As such, $\mathbf{A}_{t,t+1}$ is positive definite automatically, and thus the liner TD would converge to the TD fixed point. Next, in the case (i) we have:
	\begin{equation}\label{eq:TD_A_t}\begin{aligned}
			\mathbf{A}_t&=(\mathbf{X}+\mathbf{E})^\top \mathbf{D} (\mathbf{X}+\mathbf{E}) - (\mathbf{X}+\mathbf{E})^\top \mathbf{D} \gamma \mathbf{P} \mathbf{X}\\
			&=\mathbf{A}+\mathbf{X}^\top\mathbf{D}\mathbf{E}+\mathbf{E}^\top\mathbf{D}\mathbf{X}+\mathbf{E}^\top\mathbf{D}\mathbf{E}-\mathbf{E}^\top\mathbf{D} \gamma \mathbf{P} \mathbf{X}\\
			&=(\mathbf{X}+\mathbf{E})^\top \mathbf{D} (\mathbf{I} - \gamma \mathbf{P}) (\mathbf{X}+\mathbf{E}) + (\mathbf{X}+\mathbf{E})^\top \mathbf{D} \gamma \mathbf{P} \mathbf{E}\\
			&=\mathbf{A}_{t,t+1} + \gamma (\mathbf{X}+\mathbf{E})^\top \mathbf{D} \mathbf{P} \mathbf{E}\\
			&=\mathbf{A}_{t,t+1} + \gamma (\mathbf{X}^\top\mathbf{D} \gamma \mathbf{P} \mathbf{E} + \mathbf{E}^\top\mathbf{D} \gamma \mathbf{P} \mathbf{E})
	\end{aligned}\end{equation}	
	Similarly, in the case (ii), we can also attain:		
	\begin{equation}\label{eq:TD_A_t+1}\begin{aligned}
			\mathbf{A}_{t+1}&=\mathbf{X}^\top \mathbf{D} \mathbf{X} - \mathbf{X}^\top \mathbf{D} \gamma \mathbf{P} (\mathbf{X} + \mathbf{E})\\
			&=\mathbf{A}-\gamma \mathbf{X}^\top\mathbf{D}  \mathbf{P} \mathbf{E}
	\end{aligned}\end{equation}	
	We know that the positive definiteness of $\mathbf{A}$ and $\mathbf{A}_{t,t+1}$ is only determined by the positive definiteness of the inner matrix $\mathbf{D} (\mathbf{I} - \gamma \mathbf{P})$. If we assume the positive definiteness of $\mathbf{A}$, i.e., the positive definiteness of $\mathbf{A}_{t,t+1}$ and $\mathbf{D} (\mathbf{I} - \gamma \mathbf{P})$, as $\gamma>0$, what we only need to focus on are the positive definiteness of $\mathbf{X}^\top\mathbf{D} \mathbf{P} \mathbf{E} + \mathbf{E}^\top\mathbf{D} \mathbf{P} \mathbf{E}$ and $-\mathbf{X}^\top\mathbf{D} \mathbf{P} \mathbf{E}$. If they are positive definite, TD learning will converge under their cases, respectively.
\end{proof}

\section{Sensitivity Analysis by Influence Function}\label{appendix:IF}
Next, we conduct an outlier analysis by the \textit{influence function}, a key facet in the robust statistics~\cite{huber2004robust}. The influence function characterizes the effect that the noise in particular observation has on an estimator, and can be utilized to investigate the impact of one particular state observation noise on the training of reinforcement learning algorithms. Specifically, suppose that $F_\epsilon$ is the contaminated distribution function that combines the clear data distribution $F$ and an outlier $x$. The distribution $F_{\epsilon}$ can be defined as
\begin{equation}\begin{aligned}
		F_{\epsilon}=(1-\epsilon)F+\epsilon\delta_x,
\end{aligned}\end{equation}
where $\delta_x$ is a probability measure assigning probability 1 to ${x}$. Let $\hat{\theta}$ be a regression estimator. The influence function of $\theta$ at $F$, $\psi:\mathcal{X}\rightarrow\Gamma$ is defined as
\begin{equation}\begin{aligned}
		\psi_{\hat{\theta}, F}(x)=\lim _{\epsilon \rightarrow 0} \frac{\hat{\theta}\left(F_{\epsilon}(x)\right)-\hat{\theta}(F)}{\epsilon}.
\end{aligned}\end{equation}
Mathematically, the influence function is the Gateaux derivative of $\theta$ at $F$ in the direction $\delta_x$. Owing to the fact that traditional value-based RL algorithms, e.g., DQN~\cite{mnih2015human}, can be viewed as a regression problem~\cite{fan2020theoretical}, the linear TD approximator also has a strong connection with regression problems. Based on this correlation, in the following Theorem~\ref{theorem:IF}, we quantitatively evaluate the influence function of TD learning in the case of linear function approximation.

\begin{theorem}\label{theorem:IF}(Influence Function Analysis in TD Learning with linear function approximation) Denote $d_t=\mathbf{x}_{t}-\gamma \mathbf{x}_{t+1} \in \mathbb{R}^{d}$, and $\mathbf{A} \doteq \mathbb{E}\left[\mathbf{x}_{t} d_t^{\top}\right] \in \mathbb{R}^{d \times d}$. Let $F_{\pi}$ be the data distribution generated from the environment dynamics given a policy $\pi$. Consider an outlier pair $(\mathbf{x}_t, \mathbf{x}_{t+1})$ with the reward $R_{t+1}$, the influence function $\psi$ of this pair on the estimator $\mathbf{w}$ is derived as
	\begin{equation}\label{eq:IF}\begin{aligned}
			\psi_{\mathbf{w}, F_\pi}(\mathbf{x}_t, \mathbf{x}_{t+1}) = \mathbb{E}(\mathbf{A}^{\top}\mathbf{A})^{-1} d_t \mathbf{x}_t^{\top}\mathbf{x}_t (R_{t+1} - d_t^{\top}\mathbf{w}).
	\end{aligned}\end{equation}		
\end{theorem} 

Please refer to Appendix~\ref{appendix:IF} for the proof. Theorem~\ref{theorem:IF} shows the quantitative impact of an outlier pair $(\mathbf{x}_t, \mathbf{x}_{t+1})$ on the learned parameter $\mathbf{w}$. Moreover, a corollary can be immediately obtained to make a precise comparison of the impacts of perturbations on current and next state features.

\begin{corollary}\label{corollary:IF} Given the same perturbation $\eta$ on either current or next state features, i.e., $\mathbf{x}_t$, and $\mathbf{x}_{t+1}$, at the step $t$, if we approximate $\eta \eta^\top \mathbf{x}_t$ and $\eta \eta^\top \mathbf{w}$ as $\mathbf{0}$ as $\eta$ is small enough, the following relationship between the resulting variations of influence function, $\Delta_{\mathbf{x}_t} \psi$ and $\Delta_{\mathbf{x}_{t+1}} \psi$, holds:
	\begin{equation}\begin{aligned}
			\gamma \Delta_{\mathbf{x}_t} \psi  + \Delta_{\mathbf{x}_{t+1}} \psi= 2 \gamma d_t \eta \mathbf{x}^{\top}_t (R_{t+1} - d_t^{\top}\mathbf{w}) .
	\end{aligned}\end{equation}		
\end{corollary}

We provide the proof of Corollary~\ref{corollary:IF} in Appendix~\ref{appendix:IF}. Under this equation, the sensitivity of noises on $\mathbf{x}_t$ and $\mathbf{x}_{t+1}$, measured by $\Delta_{\mathbf{x}_t} \psi$ and $\Delta_{\mathbf{x}_{t+1}} \psi$, present a trade-off relationship as their weighted sum is definite. However, there is not an ordered relationship between $\Delta_{\mathbf{x}_t} \psi$ and $\Delta_{\mathbf{x}_{t+1}} \psi$. In summary, we conclude that the sensitivity of current and next state features against perturbations is normally divergent, and the degree of sensitivity is heavily determined by the task. These conclusions are similar to those we derived in the TD convergence part. 

\begin{proof} We combine the proof of Theorem~\ref{theorem:IF} and Corollary~\ref{corollary:IF} together. The TD fixed point $\mathbf{w}_{\text{TD}}$ to the system satisfies $\mathbf{A} \mathbf{w}_{\text{TD}} = \mathbf{b}$. Thus, the TD convergence point, i.e., TD fixed point, can be attained by solving the following regression problem:
	\begin{equation}\begin{aligned}
			\min_{\mathbf{w}} \Vert \mathbf{b} - \mathbf{A} \mathbf{w} \Vert^2
	\end{aligned}\end{equation}	
	To derive the influence function, consider the contaminated distribution which puts a little more weight on the outlier pair $(\mathbf{x}_t, \mathbf{x}_{t+1})$:
	\begin{equation}\begin{aligned}
			\hat{\mathbf{w}} = \mathop{\arg\min}_{\mathbf{w}} (1-\epsilon) \mathbb{E}[(\mathbf{b} - \mathbf{A} \mathbf{w})^{\top}(\mathbf{b} - \mathbf{A} \mathbf{w})]  + 	\epsilon (y_b-x_A^{\top} \mathbf{w})^{\top} (y_b-x_A^{\top} \mathbf{w}),
	\end{aligned}\end{equation}	
	where $y_b=R_{t+1}\mathbf{x}_t$ and $x_b=d_t \mathbf{x}_t^{\top}  $. We take the first condition:
	\begin{equation}\begin{aligned}
			(1-\epsilon)\mathbb{E}(2\mathbf{A}^{\top}\mathbf{A}\mathbf{w}-2\mathbf{A}^{\top}\mathbf{b})-2\epsilon x_A (y_b - x_A^{\top} \mathbf{w})=0.
	\end{aligned}\end{equation}	
	Then we arrange this equality and obtain:
	\begin{equation}\begin{aligned}
			(1-\epsilon)\mathbb{E}(\mathbf{A}^{\top}\mathbf{A} + x_A x_A^{\top}) \mathbf{w}_\epsilon = (1-\epsilon)\mathbb{E}(\mathbf{A}^{\top}\mathbf{b})+ \epsilon x_A y_b.
	\end{aligned}\end{equation}	
	Then we take the gradient on $\epsilon$ and let $\epsilon=0$, then we have:
	\begin{equation}\begin{aligned}
			(-\mathbb{E}(\mathbf{A}^{\top}\mathbf{A})+x_A x_A^\top) \mathbf{w}_\epsilon + \mathbb{E}(\mathbf{A}^{\top}\mathbf{A}) \psi_{\mathbf{w}, F_\pi} = -\mathbb{E}(\mathbf{A}^{\top}\mathbf{b}) +x_A y_b.
	\end{aligned}\end{equation}		
	We know that under the least square estimation, the closed-form solution of $\mathbf{w}_\epsilon$ is $\mathbb{E}(\mathbf{A}^{\top}\mathbf{A})^{-1}\mathbb{E}(\mathbf{A}^{\top}\mathbf{b})$. Thus, after the simplicity, we finally attain:
	\begin{equation}\begin{aligned}
			\psi_{\mathbf{w}, F_\pi}(\mathbf{x}_t, \mathbf{x}_{t+1})& = \mathbb{E}(\mathbf{A}^{\top}\mathbf{A})^{-1} x_A (y_b - x_A^{\top} \mathbf{w})\\
			&= \mathbb{E}(\mathbf{A}^{\top}\mathbf{A})^{-1} d_t \mathbf{x}_t^{\top}\mathbf{x}_t (R_{t+1} - d_t^{\top}\mathbf{w}).
	\end{aligned}\end{equation}	
	Next, we prove the Corollary. We only need to focus on the item $  d_t \mathbf{x}_t^{\top}\mathbf{x}_t (R_{t+1} - d_t^{\top}\mathbf{w})$, which we denote as $\psi_0$. Then we use $\Delta_{x_t} \psi$ and $\Delta_{x_{t+1}} \psi$ to represent the change of $\psi$ after adding perturbations $\eta$ on $\mathbf{x}_t$ and $\mathbf{x}_{t+1}$, respectively. In particular, since we approximate $\eta \eta^\top \mathbf{x}_t$ and $\eta \eta^\top \mathbf{w}$ as $\mathbf{0}$, then we have that the change of influence function for the perturbation $\eta$ on the current state feature $\mathbf{x}_t$:
	\begin{equation}\begin{aligned}
			\Delta_{x_t} \psi&\approx(d_t + \eta) (\mathbf{x}_t^{\top}\mathbf{x}_t + 2\eta^\top \mathbf{x}_t) (R_{t+1} - d_t^{\top}\mathbf{w} - \eta^\top \mathbf{w}) -\psi_0\\
			&\approx - d_t \mathbf{x}_t^{\top}\mathbf{x}_t \eta^\top \mathbf{w} + 2d_t \eta^\top \mathbf{x}_t (R_{t+1} - d_t^{\top}\mathbf{w}) + \eta \cdot \mathbf{x}_t^{\top}\mathbf{x}_t (R_{t+1} - d_t^{\top}\mathbf{w})\\
			& = 2d_t \eta^\top \mathbf{x}_t (R_{t+1} - d_t^{\top}\mathbf{w}) - \frac{1}{\gamma} (\gamma d_t \mathbf{x}_t^{\top}\mathbf{x}_t \eta^\top \mathbf{w} - \gamma \eta \mathbf{x}_t^{\top}\mathbf{x}_t (R_{t+1} - d_t^{\top}\mathbf{w})).
	\end{aligned}\end{equation}	
	Then the influence function for the perturbation $\eta$ on the next state feature $\mathbf{x}_{t+1}$ is:
	\begin{equation}\begin{aligned}
			\Delta_{x_{t+1}} \psi&=(d_t -\gamma \eta) \mathbf{x}_t^{\top}\mathbf{x}_t (R_{t+1} - d_t^{\top}\mathbf{w} + \gamma \eta^\top \mathbf{w})- \psi_0\\
			&\approx \gamma d_t \mathbf{x}_t^{\top}\mathbf{x}_t \eta^\top \mathbf{w} - \gamma \eta \mathbf{x}_t^{\top}\mathbf{x}_t (R_{t+1} - d_t^{\top}\mathbf{w}).
	\end{aligned}\end{equation}	
	Finally, it is easy to observe that the following relationship holds:
	\begin{equation}\begin{aligned}
			\gamma \Delta_{x_t} \psi = 2 \gamma d_t \eta x^{\top}_t (R_{t+1} - d_t^{\top}\mathbf{w}) - \Delta_{x_{t+1}} \psi.
	\end{aligned}\end{equation}	
\end{proof}

\section{Experimental Setup}\label{appendix:setting}

\paragraph{Noise Strength.} We use Gaussian noise with different standard deviations. In particular, for a better presentation to compare the difference, we select proper standard deviations as 0.05, 0.1 in Cart Pole, 0.01, 0.0125 in Mountain Car, 0.01, 0.05 in Breakout and 0.05 in Qbert. For the adversarial noises, we select the perturbation sizes $\epsilon$ as 0.05, 0.1 in Cart Pole, 0.01, 0.1 in Mountain Car, 0.005, 0.01 in Breakout, and 0.005 in Qbert. 

\paragraph{Distributional Loss.} After a linear search, in the QR-DQN, We set $\kappa=1$ for the Huber quantile loss across all tasks due to its smoothness. 

\paragraph{Cart Pole} After a linear search, in the QR-DQN, we set the number of quantiles $N$ to be 20, and evaluate both DQN and QR-DQN on 200,000 training iterations.

\paragraph{Mountain Car} After a linear search, in the QR-DQN, we set the number of quantiles $N$ to be 2, and evaluate both DQN and QR-DQN on 100,000 training iterations. 

\paragraph{Breakout and Qbert} After a linear search, in the QR-DQN, we set the number of quantiles $N$ to be 200, and evaluate both DQN and QR-DQN on 12,000,000 training iterations.

\section{Extension Analysis of Perturbations on Current, Next and Both State Observation}\label{appendix:states}

We provide the results of robust performance with different perturbation states in Breakout under random and adversarial noisy state observations in Figure~\ref{Figure_breakout_rand} and \ref{Figure_breakout_adv}, respectively.

\begin{figure*}[htbp]
	\centering
	\includegraphics[width=1.0\textwidth,trim=0 0 0 0,clip]{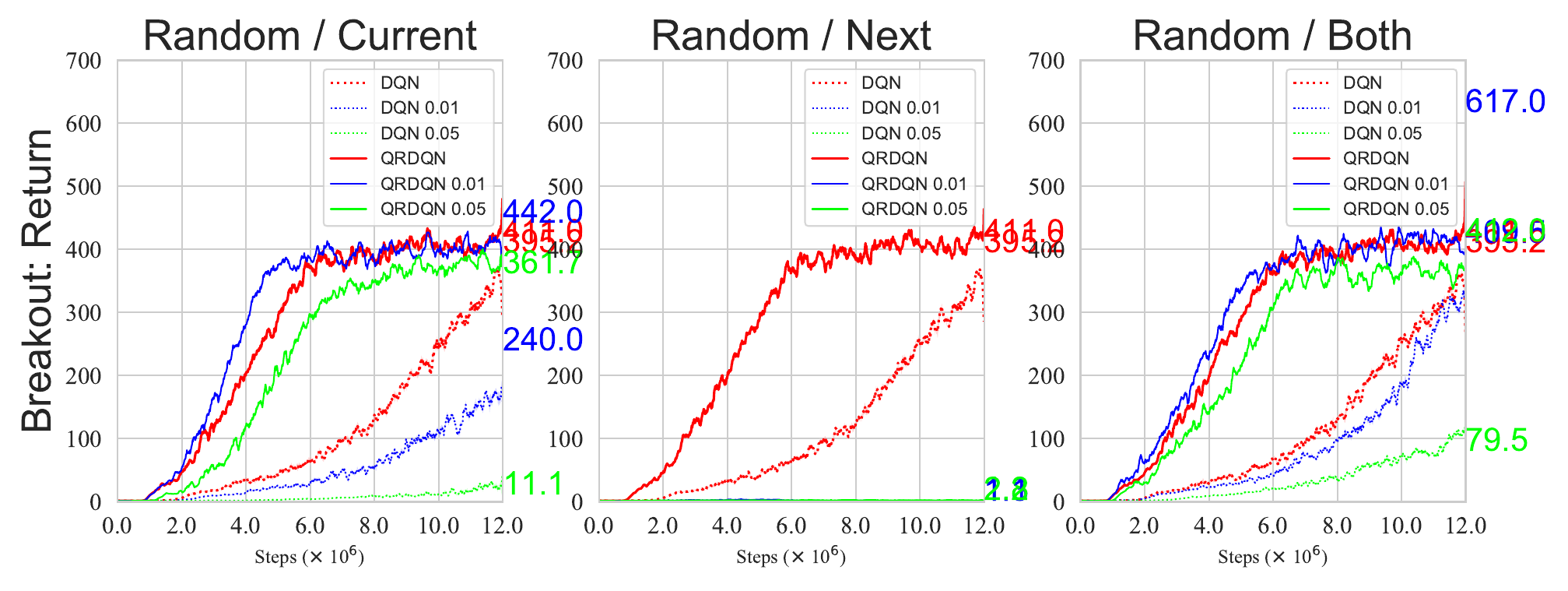}
	\caption{Average returns of DQN and QRDQN against random state observation noises on Breakout environment over 200 runs.}
	\label{Figure_breakout_rand}
\end{figure*}

In the random state observation setting as shown in Figure~\ref{Figure_breakout_rand}, it turns out that perturbations on current state observations is more robust than next state by comparing the first and second subfigures given the same perturbation strength. This empirical finding demonstrates the conclusion in Appendix~\ref{appendix:TD} and \ref{appendix:IF} that the sensitivity of current and next states are normally divergent. The result when noises are added on both states~(the right figure), seems to be most robust, which is consistent with the mildest TD convergence condition as shown in Theorem~\ref{theorem:TD}.

\begin{figure*}[htbp]
	\centering
	\includegraphics[width=1.0\textwidth,trim=0 0 0 0,clip]{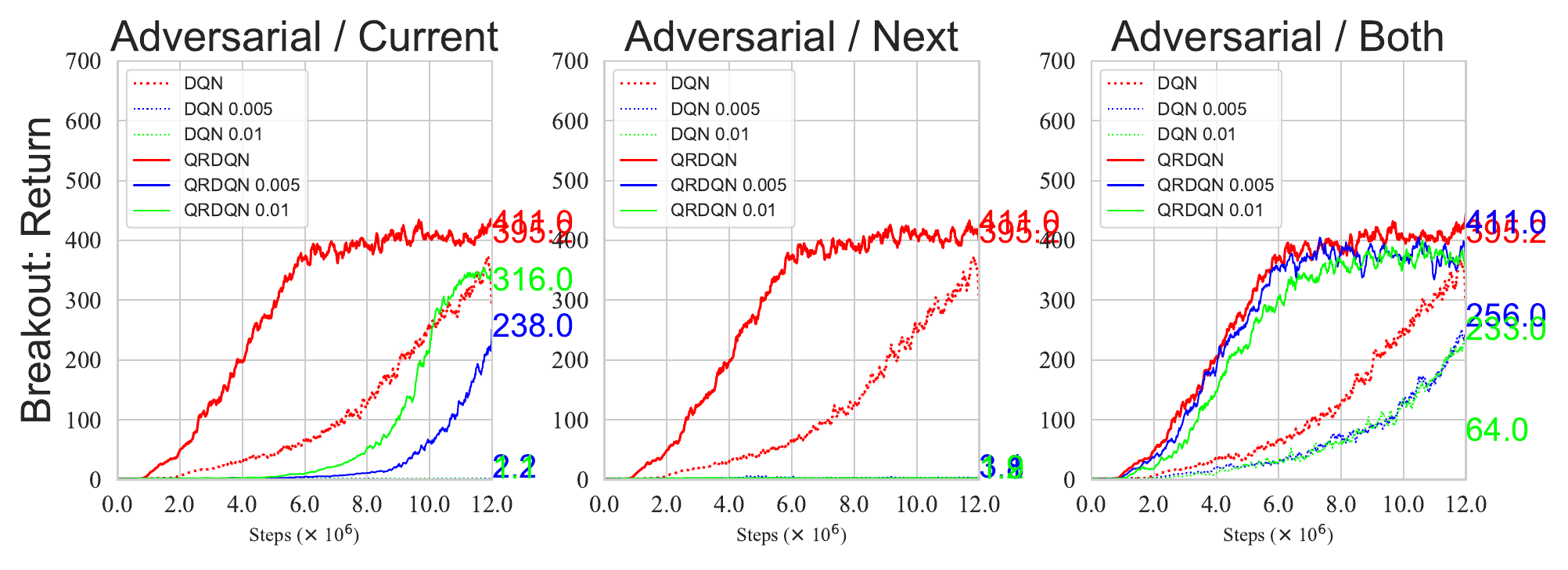}
	\caption{Average returns of DQN and QRDQN against adversarial state observation noises on Breakout environment over 200 runs.}
	\label{Figure_breakout_adv}
\end{figure*}

In the adversarial state observation setting as shown in Figure~\ref{Figure_breakout_adv}, it also turns out that perturbations on current state observations is more robust than next state by comparing the first and second subfigures given the same perturbation strength, and algorithms converge easiest. These empirical results also demonstrates the conclusion we made in Appendix~\ref{appendix:TD} and \ref{appendix:IF}.

Finally, we also provide learning curves on the typical four games, including Breakout, Qbert, Cartpole and Mountain car against \textbf{current state observations} in Figures~\ref{Figure_rand_current} and \ref{Figure_adv_current}, and against \textbf{both current and next state observation noises}, in Figures~\ref{Figure_SNDMP_rand} and \ref{Figure_SNDMP_adv}. 

\begin{figure*}[t!]
	\centering\includegraphics[width=0.9\textwidth,trim=0 0 0 0,clip]{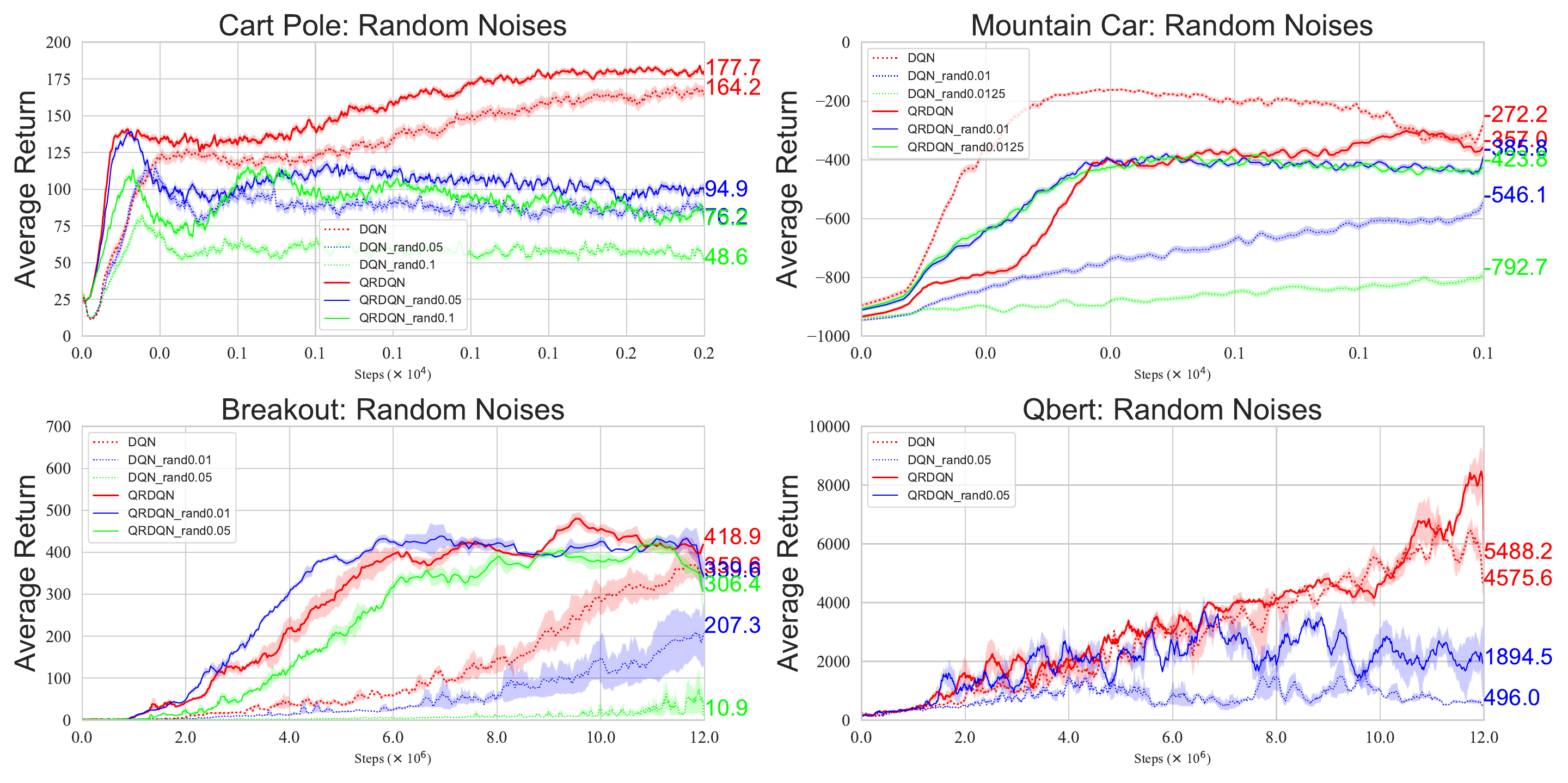}
	\caption{Average returns of DQN and QRDQN against \textbf{random} state observation noises across four games. \textbf{randX} in the legend indicates random state observations with the standard deviation \textbf{X}.}
	\label{Figure_rand_current}
\end{figure*}

\begin{figure*}[htbp]
	\centering\includegraphics[width=0.95\textwidth,trim=0 0 0 0,clip]{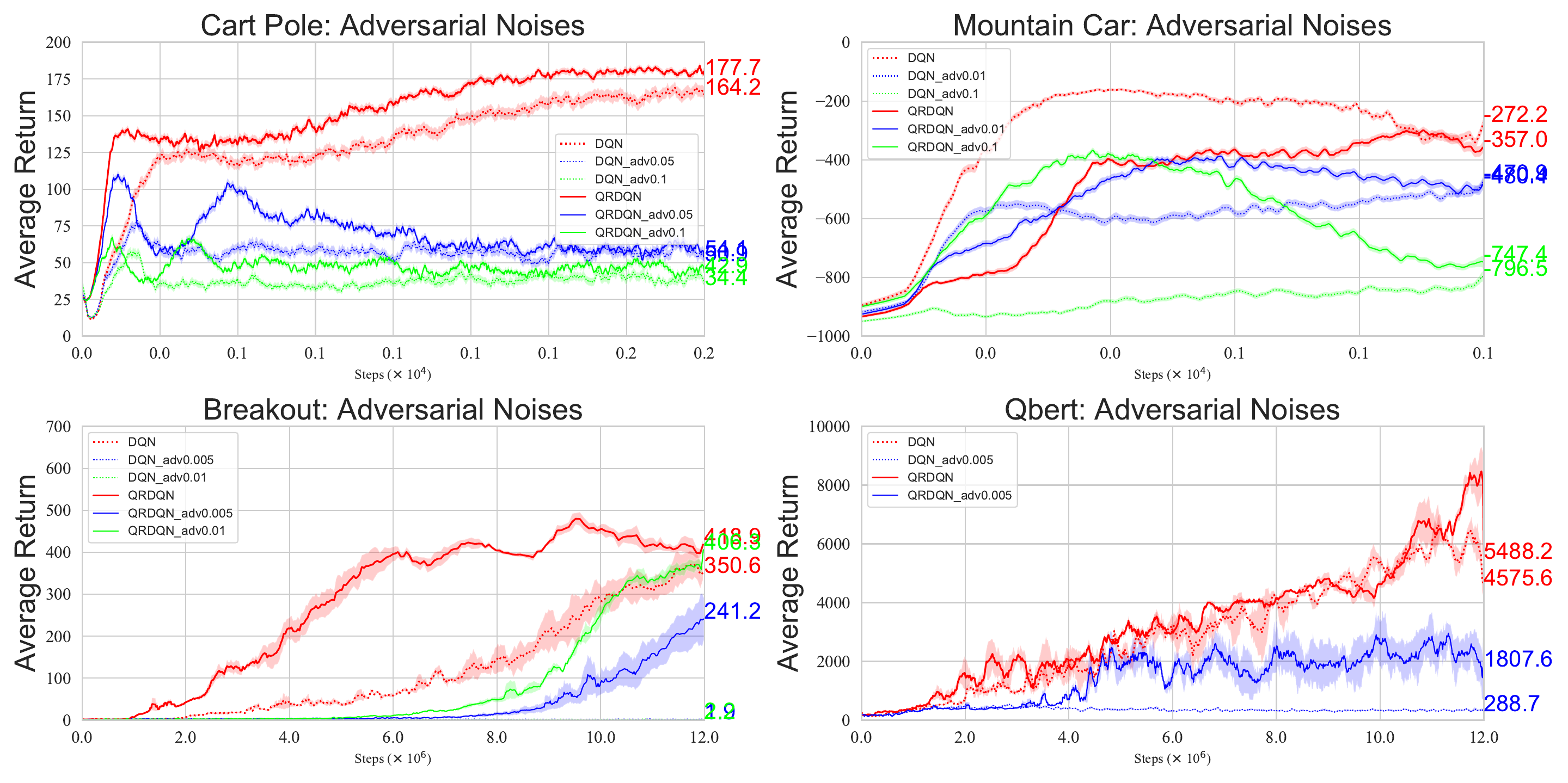}
	\caption{Average returns of DQN and QRDQN against \textbf{adversarial} state observation noises across four games. \textbf{advX} in the legend indicates random state observations with the perturbation size $\epsilon$ \textbf{X}.}
	\label{Figure_adv_current}
\end{figure*}

\begin{figure*}[htbp]
	\centering\includegraphics[width=1.0\textwidth,trim=0 0 0 0,clip]{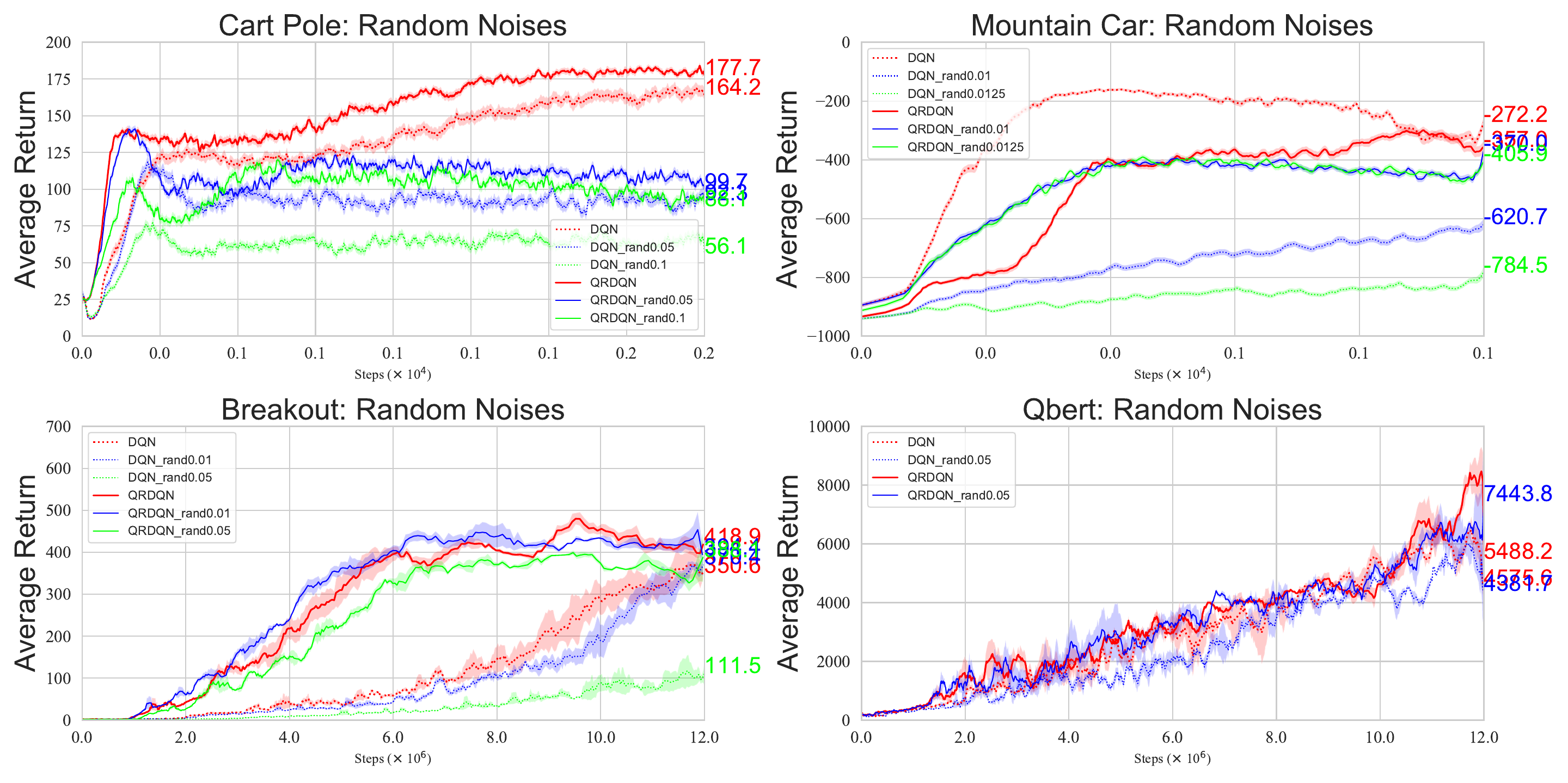}
	\caption{Average returns of DQN and QRDQN against \textbf{random} state observation noises \textbf{current and next states} across four games. \textbf{randX} in the legend indicates random state observations with the standard deviation \textbf{X}. Both QRDQN~(solid lines) and DQN~(dashed lines) converge and the convergence level is determined by the perturbation strength.}
	\label{Figure_SNDMP_rand}
\end{figure*}

\begin{figure*}[htbp]
	\centering\includegraphics[width=0.9\textwidth,trim=0 0 0 0,clip]{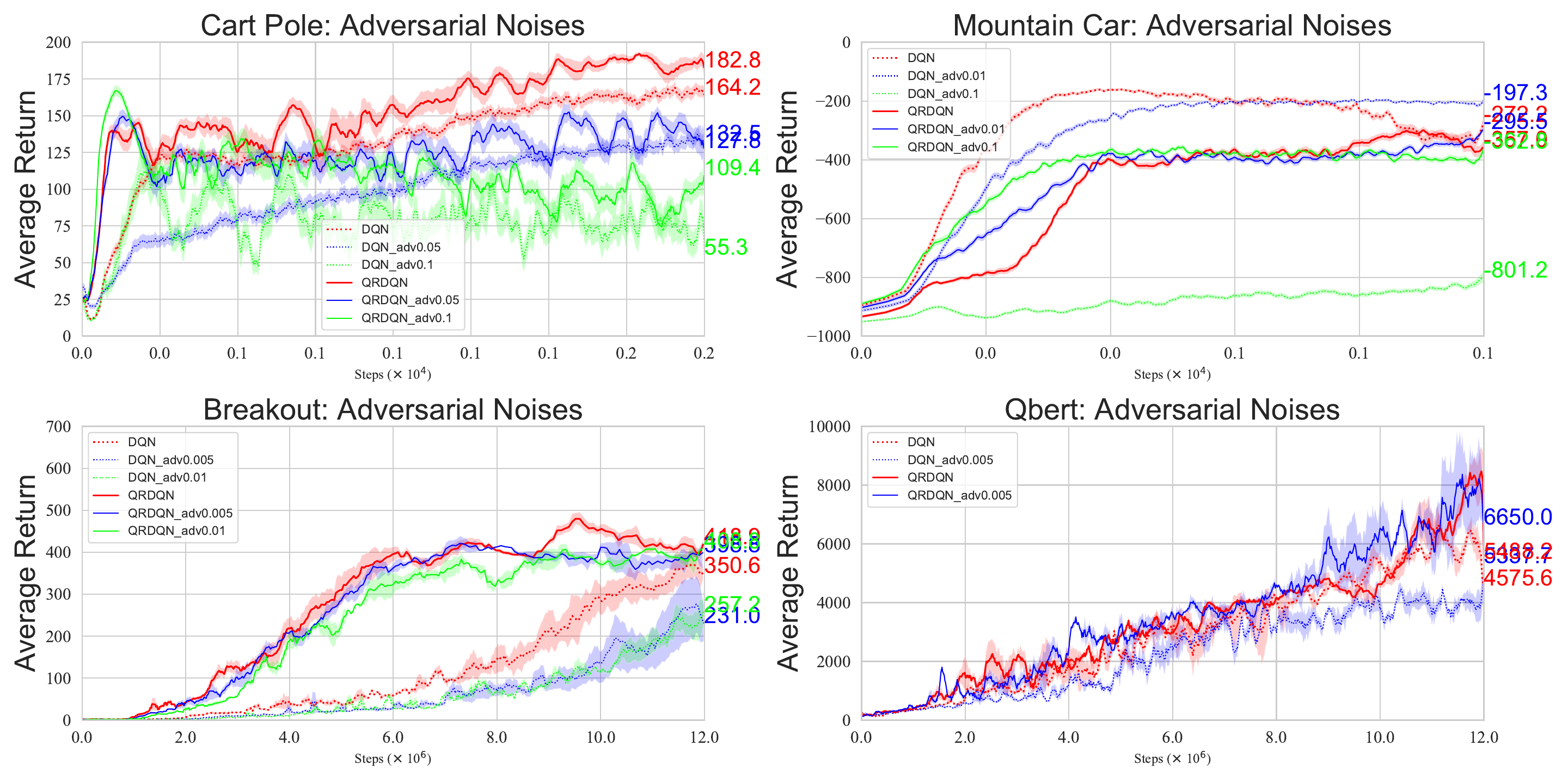}
	\caption{Average returns of DQN and QRDQN against \textbf{adversarial} state observation noises on \textbf{current and next states} across four games. \textbf{advX} in the legend indicates adversarial state observations with the perturbation size $\epsilon$ as \textbf{X}. Both QRDQN~(solid lines) and DQN~(dashed lines) converge to different optimums determined by the perturbation strength $\epsilon$.}
	\label{Figure_SNDMP_adv}
\end{figure*}

\clearpage
\section{Discussion about More Adversarial Attacks}\label{appendix:attacks}

We are investigating more advanced adversarial attacks to further demonstrate the robustness advantage of distributional RL algorithms. \cite{zhang2020robust} proposed Robust SARSA (RS) attack and Maximal Action Difference (MAD) attack, however, these two advanced attacked are specifically designed for PPO algorithm. Meanwhile, Stochastic gradient Langevin dynamics (SGLD) and convex relaxation attacks proposed by~\cite{zhang2020robust} are for DDPG algorithm. PGD attacks, serveing as the most natural attack for value-based RL algorithms, are leveraged to evaluate SA-DQN algorithm. We also probe the training robustness of distributional RL algorithms under the more advanced \textbf{MAD attack}~\cite{zhang2020robust} on Ant, where Figure~\ref{Figure_mujoco_mad} still suggests a similar robustness result of distributional RL.

\begin{figure*}[htbp]
	\centering
	\begin{subfigure}[t]{0.45\textwidth}
		\centering
		\includegraphics[width=\textwidth,trim=10 10 0 10,clip]{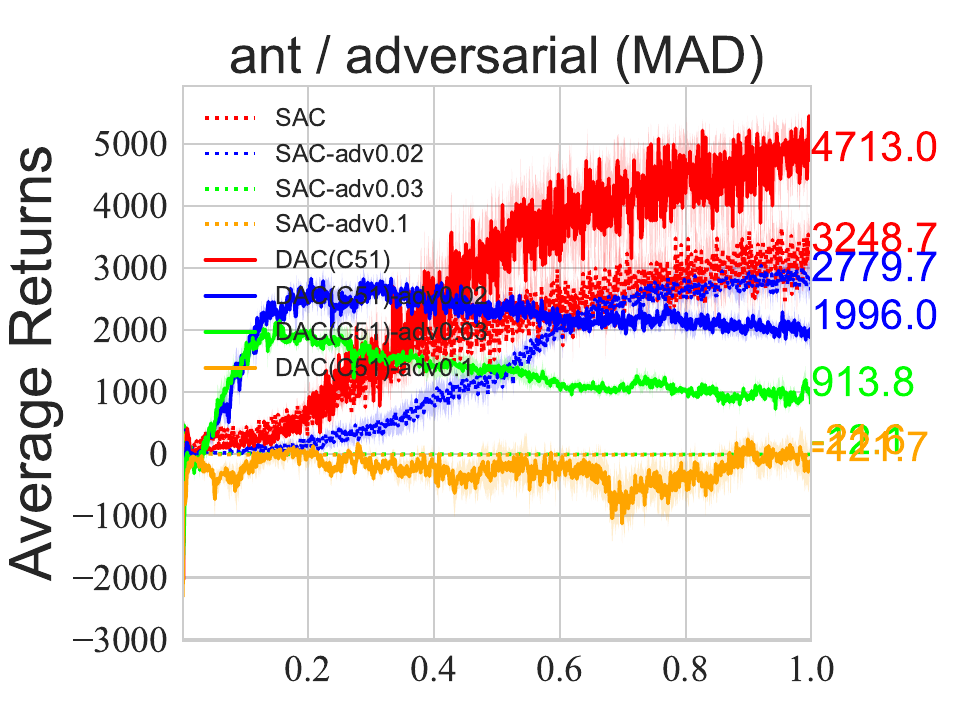}
	\end{subfigure}
	\begin{subfigure}[t]{0.45\textwidth}
		\centering
		\includegraphics[width=\textwidth,trim=0 10 0 10,clip]{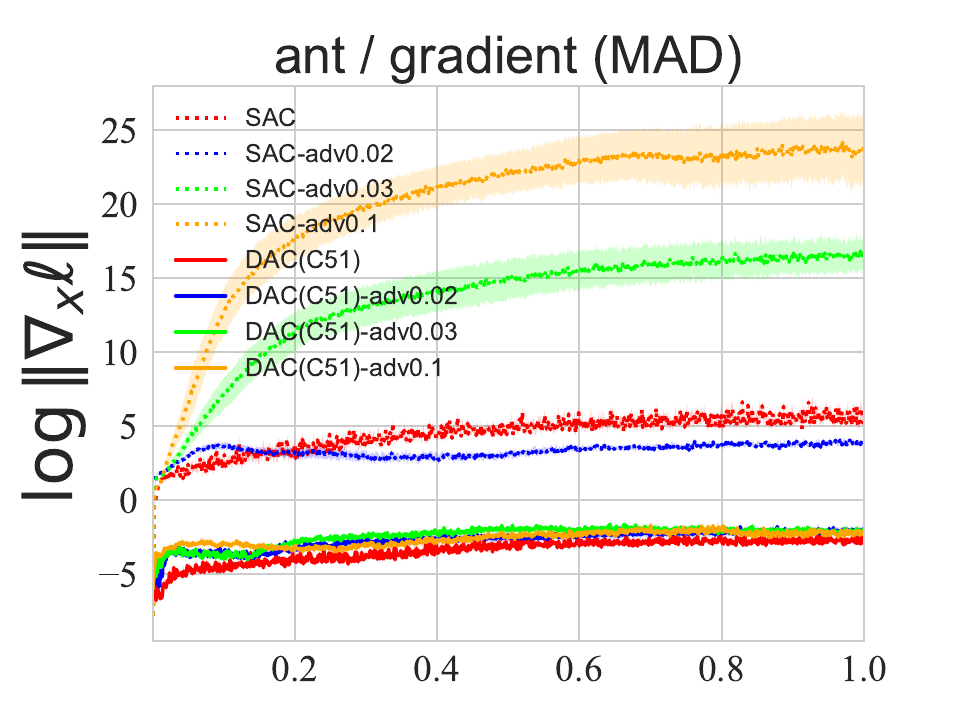}
	\end{subfigure}
	\caption{Robustness on MAD attack on Ant.}
	\label{Figure_mujoco_mad}
\end{figure*}

\section{Experiments on D4PG}\label{appendix:D4PG}

\begin{figure*}[htbp]
	\centering
	\begin{subfigure}[t]{0.45\textwidth}
		\centering
		\includegraphics[width=\textwidth,trim=0 0 0 0,clip]{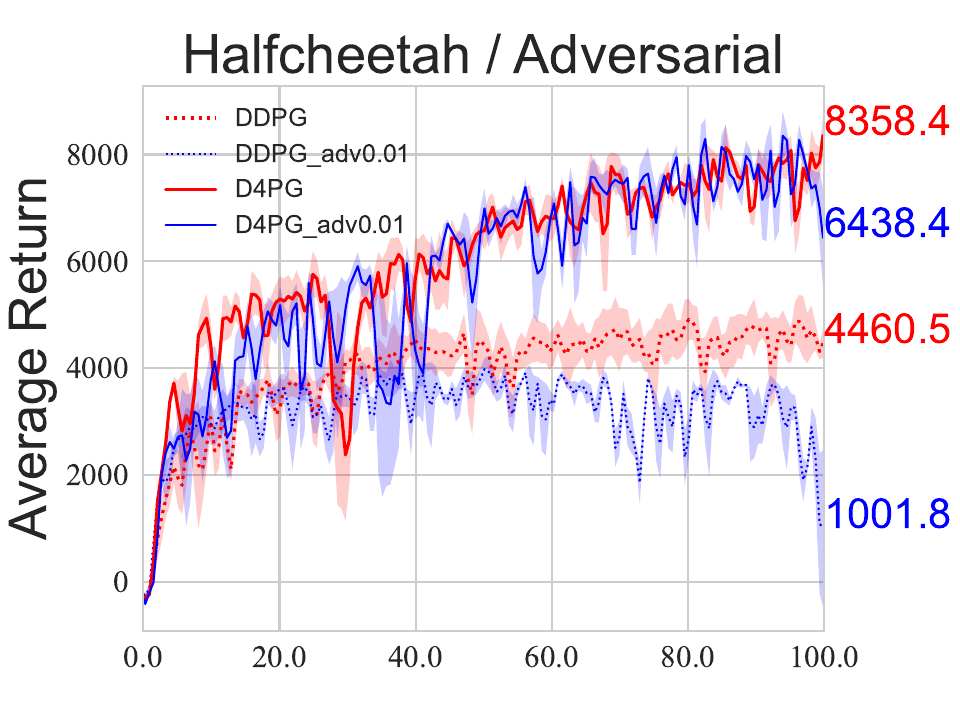}
	\end{subfigure}
	\caption{Average returns of DDPG and D4PG against \textbf{adversarial} state noises on Halfcheetah. }
	\label{Figure_mujoco_grad_newgame}
\end{figure*}

\end{document}